\documentclass[twoside]{article}
\usepackage{ecj,palatino,epsfig,latexsym}

\usepackage{amsmath,amssymb,amsthm}
\usepackage{xspace}
\usepackage{url}
\usepackage[ruled,noline,linesnumbered]{algorithm2e}
\SetKwFor{Function}{function}{is}{end function}
\SetKwFor{Subroutine}{subroutine}{is}{end subroutine}

\SetKw{Input}{input}
\SetKw{Output}{output}
\SetKw{Initialization}{initialization}
\SetKwProg{Do}{}{ do}{end}
\SetKwProg{uDo}{}{ do}{}
\SetKwProg{Otherwise}{otherwise}{ do}{end}
\SetAlgoCaptionLayout{small}

\usepackage{graphicx,subfigure}
\usepackage{tikz}
\usepackage{pgfplots}
\usepackage{pgfplotstable}
\usetikzlibrary{arrows}

\tikzset{every mark/.append style={scale=0.60}}

\usepgfplotslibrary{external}
\newtheorem{theorem}{Theorem}
\newtheorem{lemma}[theorem]{Lemma}

\newtheorem{corollary}[theorem]{Corollary}

\theoremstyle{definition}

\theoremstyle{remark}
\newtheorem{remark}{Remark}

\newcommand{\EA}{(1+1)~EA\xspace}
\newcommand{\muea}{($\mu$+$1$)~EA\xspace}

\newcommand{\mlEA}{($\mu$+$\lambda$)~EA\xspace}

\newcommand{\mlGA}{($\mu$+$\lambda$)~GA\xspace}

\newcommand{\TwoGA}{(2+1)~GA\xspace}

\newcommand{\ONEMAX}{\textsc{OneMax}\xspace}

\newcommand{\E}[1]{\mathrm{E}\mathord{\left(#1\right)}}

\newcommand{\PCal}{\ensuremath{\mathcal{P}}\xspace}
\newcommand{\R}{\ensuremath{{\mathbb R}}}
\newcommand{\N}{\ensuremath{{\mathbb N}}}

\DeclareMathOperator{\Prob}{Prob}

\newcommand{\ceil}[1]{\left\lceil#1\right\rceil}
\newcommand{\ones}[1]{\left|#1\right|_1}

\newcommand{\Hyp}[3]{\mathrm{Hyp}(#1, #2, #3)}

\newcommand{\maxk}{\max_k\left(\frac{(pn)^k}{k!k!}\right)}

\newcommand{\ignore}[1]{}

\renewcommand{\epsilon}{\varepsilon}

\pgfplotsset{
cycle list={%
    {blue,mark=*},
    {red, mark=square*},
    {brown!50!black, mark=star},
    {green!30!black, mark=diamond},
    {blue,mark=pentagon}
    {red, mark=x},
    {blue,mark=x},
    },
}

\begin{document}


\title{How Crossover Speeds Up Building-Block Assembly in~Genetic~Algorithms}

\author{\name{\bf Dirk Sudholt}\\
        \addr{Department of Computer Science,}
        \addr{University of Sheffield, United Kingdom}
}

\maketitle

\begin{abstract}
We re-investigate a fundamental question: how effective is crossover in Genetic Algorithms in combining building blocks of good solutions? Although this has been discussed controversially for decades, we are still lacking a rigorous and intuitive answer. We provide such answers for royal road functions and \ONEMAX, where every bit is a building block. For the latter we show that using crossover makes \emph{every} ($\mu$+$\lambda$)~Genetic Algorithm at least twice as fast as the fastest evolutionary algorithm using only standard bit mutation, up to small-order terms and for moderate~$\mu$ and~$\lambda$.
Crossover is beneficial because it effectively turns fitness-neutral mutations into improvements by combining the right building blocks at a later stage. Compared to mutation-based evolutionary algorithms, this makes multi-bit mutations more useful. Introducing crossover changes the optimal mutation rate on \ONEMAX from $1/n$ to ${(1+\sqrt{5})/2 \cdot 1/n \approx 1.618/n}$. This holds both for uniform crossover and $k$\nobreakdash-point crossover. Experiments and statistical tests confirm that our findings apply to a broad class of building-block functions.
\end{abstract}


\section{Introduction}

Ever since the early days of genetic algorithms (GAs), researchers have wondered when and why crossover is an effective search operator.
It has been folklore that crossover is useful if it can combine building blocks, i.\,e., schema of high fitness, to give better solutions~\cite{Mitchell1992}. But, as Watson and Jansen~\cite{Watson2007} put it, \emph{there has been a considerable difficulty in demonstrating this rigorously and intuitively}.

Many attempts at understanding crossover have been made in the past.
Mitchell, Forrest, and Holland~\cite{Mitchell1992} presented so-called \emph{royal road} functions as an example where, supposedly, genetic algorithms outperform other search algorithms due to the use of crossover. Royal roads divide a bit string into disjoint blocks. Each block makes a positive contribution to the fitness in case all bits therein are set to~1. Blocks thus represent schemata, and all-ones configurations are building blocks of optimal solutions. However, the same authors later concluded that simple randomized hill climbers performed better than GAs~\cite{Forrest1993,Mitchell1994}.


The role of crossover has been studied from multiple angles, including algebra~\cite{Rowe2002}, Markov chain models~\cite{Vose1999}, infinite population models and dynamical systems (see~\cite[Chapter~6]{DeJong2006} for an overview) and statistical mechanics (see, e.\,g.\ \cite[Chapter~11]{Eiben2007}).

Also in biology the role of crossover is far from settled.
In population genetics, exploring the advantages of recombination, or sexual reproduction, is a famous open question~\cite{Barton98} and has been called ``the queen of problems in evolutionary biology'' by Graham Bell~\cite{Bell82} and others. Evolutionary processes were found to be harder to analyze than those using only asexual reproduction as they represent quadratic dynamical systems~\cite{AroraRV94,RabaniRS98}.

Recent work in population genetics has focussed on studying the ``speed of adaptation'', which describes the efficiency of evolution, in a similar vein to research in evolutionary computation~\cite{weissman_limits_2012,Weissman2010}.
Furthermore, a new theory of mixability has been proposed recently from the perspective of theoretical computer science~\cite{Livnat2008,Livnat2010}, arguing that recombination favours individuals that are good ``mixers'', that is, individuals that create good offspring when being recombined with others.

Several researchers recently and independently reported empirical observations that using crossover improves the performance of evolutionary algorithms (EAs) on the simple function \ONEMAX$(x) = \sum_{i=1}^n x_i$~\cite{LassigPersonal,Rowe2013Chapter}, but were unable to explain why. The fact that even settings as simple as \ONEMAX are not well understood demonstrates the need for a solid theory and serves as motivation for this work.

Runtime analysis has become a major area of research that can give rigorous evidence and proven theorems~\cite{BookNeuWit,Auger2011,Jansen2013}. However, studies so far have eluded the most fundamental setting of building-block functions. Crossover was proven to be superior to mutation only on constructed artificial examples like $\mathrm{Jump}_k$~\cite{Jansen2002,Koetzing2011a} and ``Real Royal Road'' functions~\cite{Jansen2005c,Storch2004}, the H-IFF problem~\cite{Dietzfelbinger2003}, coloring problems inspired by the Ising model from physics~\cite{Fischer2005,Sudholt2005}\footnote{For bipartite graphs, the problem is equivalent to the classical Graph Coloring problem with 2 colors.}, computing unique input-output sequences for finite state machines~\cite{Lehre2011b}, selected problems from multi-objective optimization~\cite{Qian2013}, and the all-pairs shortest path problem~\cite{Doerr2012,Sudholt2011a,Neumann2010b}.
H-IFF~\cite{Dietzfelbinger2003} and the Ising model on trees~\cite{Sudholt2005} consist of hierarchical building blocks. But none of the above papers addresses single-level building blocks in a setting as simple as royal roads.

Watson and Jansen~\cite{Watson2007} presented a constructed building-block function and proved exponential performance gaps between EAs using only mutation and a GA. However, the definition of the internal structure of building blocks is complicated and artificial, and they used a tailored multi-deme GA to get the necessary diversity. With regard to the question on how GAs combine building blocks, their approach does not give the intuitive explanation one is hoping for.

This paper presents such an intuitive explanation, supported by rigorous analyses. We consider royal roads and other functions composed of building blocks, such as monotone polynomials. $\ONEMAX(x) = \sum_{i=1}^n x_i$ is a special case where every bit is a building block. We give rigorous proofs for \ONEMAX and show how the main proof arguments transfer to broader classes of building-block functions. Experiments support the latter.

Our main results are as follows.
\begin{enumerate}
\item We show in Section~\ref{sec:uniform-crossover} that on \ONEMAX \emph{every} \mlGA with uniform crossover and standard bit mutation is at least twice as fast as \emph{every} evolutionary algorithm (EA) that only uses standard bit mutations (up to small-order terms). More precisely, the dominating term in the expected number of function evaluations decreases from $e \cdot n \ln n$ to $e/2 \cdot n \ln n$.
    This holds provided that the parent population and offspring population sizes $\mu$ and $\lambda$ are moderate, so that the inertia of a large population does not slow down exploitation.
    The reason for this speedup is that the GA can store a neutral mutation (a mutation not altering the parent's fitness) in the population, along with the respective parent. It can then use crossover to combine the good building blocks between these two individuals, improving the current best fitness. In other words, crossover can capitalize on mutations that have both beneficial and disruptive effects on building blocks.
\item The use of uniform crossover leads to a shift in the optimal mutation rate on \ONEMAX. We demonstrate this in Section~\ref{sec:optimal-mutation-rate} for a simple ``greedy'' \TwoGA that always selects parents among the current best individuals. While for mutation-based EAs $1/n$ is the optimal mutation rate~\cite{Witt2013}, the greedy \TwoGA has an optimal mutation rate of $(1+\sqrt{5})/2 \cdot 1/n \approx 1.618/n$ (ignoring small-order terms). This is because introducing crossover makes neutral mutations more useful and larger mutation rates increase the chance of a neutral mutation. Optimality is proved by means of a matching lower bound on the expected optimization time of the greedy \TwoGA that applies to \emph{all} mask-based crossover operators (where each bit value is taken from either parent).
    Using the optimal mutation rate, the expected number of function evaluations is~$1.19 n \ln n \pm O(n \log \log n)$.
\item These results are not limited to uniform crossover or the absence of linkage. Section~\ref{sec:k-point-crossover} shows that the same results hold for GAs using $k$\nobreakdash-point crossover, for arbitrary~$k$, under slightly stronger conditions on $\mu$ and $\lambda$, if the crossover probability $p_c$ is set to an appropriately small value.
\item The reasoning for \ONEMAX carries over to other functions with a clear building block structure. Experiments in Section~\ref{sec:extensions} reveal similar performance differences as on \ONEMAX for royal road functions and random polynomials with unweighted, positive coefficients. This is largely confirmed by statistical tests. There is evidence that findings also transfer to weighted building-block functions like linear functions, provided that the population can store solutions with different fitness values and different building blocks until crossover is able to combine them. This is not the case for the greedy \TwoGA, but a simple (5+1)~GA is significantly faster on random linear functions than the optimal mutation-based EA for this class of functions, the \EA~\cite{Witt2013}.
\end{enumerate}
The first result, the analysis for uniform crossover, is remarkably simple and intuitive. It gives direct insight into the working principles of GAs. Its simplicity also makes it very well suited for teaching purposes.


This work extends a preliminary conference paper~\cite{Sudholt2012b} with parts of the results, where results were restricted to one particular GA, the greedy \TwoGA.
This extended version presents a general analytical framework that applies to all \mlGA{}s, subject to mild conditions, and includes the greedy \TwoGA as a special case. To this end, we provide tools for analyzing parent and offspring populations in \mlGA{}s, which we believe are of independent interest.

Moreover, results for $k$-point crossover have been improved. The leading constant in the upper bound for $k$-point crossover in~\cite{Sudholt2012b} was by an additive term of $\frac{2c}{3 + 3c}$ larger than that for uniform crossover, for mutation rates of $c/n$. This left open the question whether $k$-point crossover is as effective as uniform crossover for assembling building blocks in \ONEMAX. Here we provide a new and refined analysis, which gives an affirmative answer, under mild conditions on the crossover probability.

\subsection{Related Work}

K{\"o}tzing, Sudholt, and Theile~\cite{Koetzing2011a} considered the search behaviour of an idealized GA on \ONEMAX, to highlight the potential benefits of crossover under ideal circumstances. If a GA was able to recombine two individuals with equal fitness that result from independent evolutionary lineages, the fitness gain can be of order $\Omega(\sqrt{n})$. The idealized GA would therefore be able to optimize \ONEMAX in expected time~$O(\sqrt{n})$~\cite{Koetzing2011a}. However, this idealization cannot reasonably be achieved in realistic EAs with common search operators, hence the result should be regarded an academic study on the \emph{potential} benefit of crossover.

A related strand of research deals with the analysis of the Simple~GA on \ONEMAX. The Simple~GA is one of the best known and best researched GAs in the field. It uses a generational model where parents are selected using fitness-proportional selection and the generated offspring form the next population. Neumann, Oliveto and Witt~\cite{Neumann2009c} showed that the Simple~GA without crossover with high probability cannot optimize \ONEMAX in less than exponential time. The reason is that the population typically contains individuals of similar fitness, and then fitness-proportional selection is similar to uniform selection. Oliveto and Witt~\cite{Oliveto2013a} extended this result to uniform crossover: the Simple~GA with uniform crossover and population size $\mu \le n^{1/8-\varepsilon}$, $\varepsilon > 0$, still needs exponential time on \ONEMAX. It even needs exponential time to reach a solution of fitness larger than $(1+c)\cdot n/2$ for an arbitrary constant $c > 0$.
In~\cite{Oliveto2013} the same authors relaxed their condition on the population size to $\mu \le n^{1/4-\varepsilon}$. Their work does not exclude that crossover is advantageous, particularly since under the right circumstances crossover may lead to a large increase in fitness (cf.~\cite{Koetzing2011a}). But if there is an advantage, it is not noticeable as the Simple~GA with crossover still fails badly on \ONEMAX.

One year after~\cite{Sudholt2012b} was published, Doerr, Doerr, and Ebel~\cite{Doerr2013a} presented a groundbreaking result: they designed an EA that was proven to optimise \ONEMAX (and any simple transformation thereof) in time $O(n \sqrt{\log n})$. This is a spectacular result as all black-box search algorithms using only unbiased unary operators---operators modifying one individual only, and not exhibiting any inherent search bias---need time $\Omega(n \log n)$ as shown by Lehre and Witt~\cite{Lehre2012}. So their EA shows that crossover can lower the expected running time by more than a constant factor. They call their algorithm a \mbox{(1+($\lambda$, $\lambda$))~EA:} starting with one parent, it first creates $\lambda$ offspring by mutation, with a random and potentially high mutation rate. Then it selects the best mutant, and crosses it $\lambda$ times with the original parent, using parameterized uniform crossover (the probability of taking a bit from the first parent is not always $1/2$, but a parameter of the algorithm). This leads to a number of $O(n \sqrt{\log n})$ expected function evaluations, which can be further decreased to~$O(n)$ with a scheme adapting $\lambda$ according to the current fitness.

The (1+($\lambda$, $\lambda$))~EA from~\cite{Doerr2013a} is very cleverly designed to work efficiently on \ONEMAX and similar functions. It uses a non-standard EA design because of its two phases of environmental selection. Other differences are that mutation is performed before crossover, and mutation is not fully independent for all offspring: the number of flipping bits is a random variable determined as for standard bit mutations, but the same number of flipping bits is then used in all offspring. The focus of this work is different as our goal is to understand how standard EAs operate, and how crossover can be used to speed up building-block assembly in commonly used \mlEA{}s.

\section{Preliminaries}

We measure the performance of the algorithm with respect to the number of function evaluations performed until an optimum is found, and refer to this as \emph{optimization time}. For steady-state algorithms this equals the number of generations (apart from the initialization), and for EAs with offspring populations such as \mlEA{}s or \mlGA{}s the optimization time is by a factor of $\lambda$ larger than the number of generations. Note that the number of generations needed to optimize a fitness function can often be easily decreased by using offspring populations or parallel evolutionary algorithms~\cite{Lassig2013a}. But this significantly increases the computational effort within one generation, so the number of function evaluations is a more fair and widely used measure.

Looking at function evaluations is often motivated by the fact that this operation dominates the execution time of the algorithm. Then the number of function evaluations is a reliable measure for wall clock time. However, the wall clock time might increase when introducing crossover as an additional search operator. Also when increasing the mutation rate, more pseudo-random numbers might be required. Jansen and Zarges~\cite{Jansen2011} point out a case where this effect leads to a discrepancy between the number of function evaluations and wall clock time. This concern must be taken seriously when aiming at reducing wall clock time. However, each implementation must be checked individually in this respect~\cite{Jansen2011}. Therefore, we keep this concern in mind, but still use the number of function evaluations in the following.

\section{Uniform Crossover Makes ($\mu$+$\lambda$)~EAs Twice as Fast}
\label{sec:uniform-crossover}

We show that, under mild conditions, every \mlGA{} is at least twice as fast as its counterpart without crossover. For the latter, that is, evolutionary algorithms using only standard bit mutation, the author recently proved the following lower bound on the running time of a very broad class of \emph{mutation-based EAs}~\cite{Sudholt2012c}. It covers all possible selection mechanisms, parent or offspring populations, and even parallel evolutionary algorithms. We slightly rephrase this result.
\begin{theorem}[Sudholt~\cite{Sudholt2012c}]
\label{the:lower-onemax}
Let $n \ge 2$.
Every EA that uses only standard bit mutation with mutation rate~$p$ to create new solutions has expected optimization time at least
\[
\frac{\min\{\ln n, \ln(1/(p^2n))\} - \ln \ln n - 3}{p(1-p)^n}
\]
on \ONEMAX and every other function with a unique optimum, if ${2^{-n/3} \le p \le \frac{1}{\sqrt{n} \log n}}$. If $p=c/n$, $c > 0$ constant, this is at least
\[
c \cdot e^{c} \cdot n \ln n \cdot (1-o(1)).
\]
\end{theorem}
In fact, for \ONEMAX the author proved that among all evolutionary algorithms that start with one random solution and only use standard bit mutations the expected number of function evaluations is minimized by the simple \EA~\cite[Theorem~13]{Sudholt2012c}. Also the mutation rate $p=1/n$ is the best possible choice for \ONEMAX, leading to a lower bound of
\[
en \ln n - en \ln \ln n - 3en.
\]
For the special case of $p=1/n$, Doerr, Fouz, and Witt~\cite{Doerr2011c} recently improved the above bound towards $en \ln n - O(n)$.

We show that for a range of \mlEA{}s, as defined in the following, introducing uniform crossover can cut the dominant term of the running time in half, for the standard mutation rate $p=1/n$.

The only requirement on the parent selection mechanism is that selection does not favor inferior solutions over fitter ones. Formally, for maximizing a fitness function~$f$,
\begin{equation}
\label{eq:fitness-respectful-selection}
\forall x, y \colon f(x) \ge f(y) \Rightarrow \Prob(\text{select~$x$}) \ge \Prob(\text{select~$y$}).
\end{equation}
This in particular implies that equally fit solutions are selected with the same probability. Condition~\eqref{eq:fitness-respectful-selection} is satisfied for all common selection mechanisms: uniform selection, fitness-proportional selection, tournament selection, cut selection, and rank-based mechanisms.

The class of \mlEA{}s covered in this work is defined in Algorithm~\ref{alg:Scheme-GA}. All \mlEA{}s therein create $\lambda$ offspring through crossover and mutation, or just mutation, and then pick the best out of the $\mu$ previous search points and the $\lambda$ new offspring.
\begin{algorithm}
\label{alg:Scheme-GA}
Initialize population $\PCal$ of size $\mu \in \N$ u.\,a.\,r.\\
\While{true}{%
    Let $\PCal' = \emptyset$.\\
    \For{$i = 1, \dots, \lambda$}{%
        \uDo{\upshape{}With probability $p_c$}{
            Select $x_1, x_2$ with an operator respecting~\eqref{eq:fitness-respectful-selection}.\\
            Let $y := \text{uniform crossover}(x_1, x_2)$.
        }
        \Otherwise{}{
            Select $y$ with an operator respecting~\eqref{eq:fitness-respectful-selection}.
        }
		Flip each bit in $y$ independently with probability~$p$.\\
        Add $y$ to $\PCal'$.
    }
    \label{line:tie-breaking}Let $\PCal$ contain the $\mu$ best individuals from $\PCal \cup \PCal'$; break ties towards including individuals with the fewest duplicates in $\PCal \cup \PCal'$.
}
\caption{Scheme of a \mlGA with mutation rate~$p$ and uniform crossover with crossover probability~$p_c$ for maximizing $f \colon \{0, 1\}^n \to \R$.}
\end{algorithm}

In the case of ties, we pick solutions that have the fewest duplicates among the considered search points. This strategy has already been used by Jansen and Wegener~\cite{Jansen2005c} in their groundbreaking work on Real Royal Roads; it ensures a sufficient degree of diversity whenever the population contains different search points of the same fitness.

Before stating the main result of this section, we provide two lemmas showing how to analyse population dynamics. Both lemmas are of independent interest and may prove useful in other studies of population-based EAs.

The following lemma estimates the expected time until individuals with fitness at least~$i$ take over the whole population. It generalizes Lemma~3 in~\cite{Sudholt2009}, which in turn goes back to Witt's analysis of the \muea~\cite{Witt2006}. Note that the lemma applies to arbitrary fitness functions, arbitrary values for $\mu$ and $\lambda$, and arbitrary crossover operators; it merely relies on fundamental and universal properties of cut selection and standard bit mutations.
\begin{lemma}
\label{lem:takeover}
Consider any \mlGA implementing Algorithm~\ref{alg:Scheme-GA}, with any crossover operator, on any $n$-bit fitness function. Assume the current population contains at least one individual of fitness~$i$. The expected number of function evaluations needed for the \mlGA before all individuals in its current population have fitness at least~$i$ is at most
\[
\frac{O((\mu+\lambda) \log \mu)}{(1-p_c) (1-p)^{n}}.
\]
This holds for any tie-breaking rule used in the environmental selection.
\end{lemma}
\begin{proof}
Call an individual \emph{fit} if it has fitness at least~$i$.
We now estimate the expected number of generations until the population is taken over by fit individuals, which we call the \emph{expected takeover time}. As fit individuals are always preferred to non-fit individuals in the environmental selection, the expected takeover time equals the expected number of generations until $\mu$ fit individuals have been created, starting with one fit individual.

For each offspring being created, there is a chance that the \mlGA will simply create a clone of a fit individual. This happens if, during the creation of an offspring, the \mlGA decides not to perform crossover, it
selects a fit individual as parent to be mutated, and mutation does not flip any bit. The probability for this event is at least
\[
(1 - p_c) \cdot (1-p)^n \cdot \frac{\text{number of fit individuals in population}}{\mu}
\]
since each fit individual is selected as parent with probability at least $1/\mu$.

Now we divide the run of the \mlGA into phases in order to get a lower bound on the number of fit individuals at certain time steps. The $j$-th phase, $0 \le j \le \lceil \log_5 \mu \rceil - 1$, starts with the first offspring creation in the first generation where the number of fit individuals is at least $5^j$. It ends in the first generation where this number is increased to $\min\{5^{j+1}, \mu\}$.
Let $T_j$ describe the random number of generations spent in the $j$-th phase.
Starting with a new generation with $\mu \ge 5^j$ fit individuals in the parent population, we now consider a phase of $8\mu/((1-p_c)(1-p)^n)$ offspring creations, disregarding generation bounds.

Let $N_i$ denote the random number of new fit offspring created in the phase, then
\[
\E{N_i} \ge \frac{8\mu}{(1-p_c)(1-p)^n} \cdot (1-p_c)(1-p)^n \cdot \frac{5^i}{\mu} = 8 \cdot 5^i
\]
and by classical Chernoff bounds (see, e.\,g.~\cite[Chapter~4]{Mitzenmacher2005})
\[
\Prob(N_i < 4 \cdot 5^i) \le e^{-\E{N_i}/8} \le e^{-5^i} \le e^{-1}.
\]
If $N_i < 4 \cdot 5^i$ the phase is called unsuccessful and we consider another phase of ${8\mu/((1-p_c)(1-p)^n)}$ offspring creations.
The expected waiting time for a successful phase is at most $1/(1-e^{-1})$ and the expected number of offspring creations until $N_i \ge 4 \cdot 5^i$ is at most ${8\mu/((1-p_c)(1-p)^n(1-e^{-1}))}$.

Since phases start at generation bounds, we may need to account for up to $\lambda - 1$ further offspring creations in between phases. This implies
\[
\E{T_i} \le \frac{8\mu}{(1-p_c)(1-p)^n(1-e^{-1})} + \lambda
\]
and the expected takeover time is at most
\begin{align*}
\sum_{i=0}^{\ceil{\log_5 \mu}-1} \E{T_i}
\le\;& \ceil{\log_5 \mu} \cdot \left(\frac{8\mu}{(1-p_c)(1-p)^n(1-e^{-1})} + \lambda\right)\\
=\;& \frac{O((\mu + \lambda) \log \mu)}{(1-p_c)(1-p)^n}. \qedhere
\end{align*}
\end{proof}

We also provide the following simple but handy lemma, which relates success probabilities for created offspring to the expected number of function evaluations needed to complete a generation where such an event has first happened.
\begin{lemma}
\label{lem:success-probability-lambda-offspring}
Consider any \mlGA implementing Algorithm~\ref{alg:Scheme-GA}, and assume that in each offspring creation there is a probability at least~$q$ that some specific event occurs. Then the expected number of function evaluations to complete a generation where this event first occurs is at most
\[
\lambda - 1 + \frac{1}{q}.
\]
\end{lemma}
\begin{proof}
The expected number of trials for an event with probability~$q$ to occur is $1/q$. To complete the generation, at most $\lambda - 1$ further function evaluations are required.
\end{proof}

Now we are able to prove the main result of this section.
\begin{theorem}
\label{the:upper-bound-uniform-crossover-variable-p-every-GA}
The expected optimization time of every \mlGA implementing Algorithm~\ref{alg:Scheme-GA} with $0 < p_c < 1$ constant, mutation probability~$0 < p < 1$ and $\mu \ge 2$ on \ONEMAX is at most
\begin{equation}
\label{eq:upper-bound-GAs-general-p}
\frac{\ln(n^2 p + n) + 1 + p}{p(1-p)^{n-1} \cdot (1+np)} + \frac{O((\mu + \lambda)n \log \mu)}{(1-p)^n}.
\end{equation}
If $p=c/n$, $c > 0$ constant, and $\mu, \lambda = o((\log n)/(\log \log n))$, this bound simplifies to
\begin{equation}
\label{eq:upper-bound-GAs-one-plus-o-one}
\frac{n \ln n}{c \cdot e^{-c} \cdot (1+c)} \cdot (1 + o(1)).
\end{equation}
Both statements hold for arbitrary initial populations.
\end{theorem}
The main difference between the upper bound for \mlGA{}s and the lower bound for all mutation-based EAs is an additional factor of $1+pn$ in the denominator of the upper bound. This is a factor of $2$ for $p=1/n$ and an even larger gain for larger mutation rates.

For the default value of $p=1/n$, this shows that introducing crossover makes EAs at least twice as fast as the fastest EA using only standard bit mutation. It also implies that introducing crossover makes EAs at least twice as fast as their counterparts without crossover (i.\,e.\ where {${p_c=0}$}).

%
\begin{proof}[Proof of Theorem~\ref{the:upper-bound-uniform-crossover-variable-p-every-GA}]
Bound~\eqref{eq:upper-bound-GAs-one-plus-o-one} can be derived from~\eqref{eq:upper-bound-GAs-general-p} using $(1-1/x)^{x-1} \ge 1/e$ for $x > 1$ to estimate
\[
\left(1-\frac{c}{n}\right)^{n-1}
= \left(1 - \frac{c}{n}\right)^{(n/c-1) \cdot c} \cdot \left(1 - \frac{c}{n}\right)^{c-1} \ge \frac{1}{e^c} \cdot \left(1 - \frac{c^2}{n}\right)
= e^{-c} - O(1/n)
\]
as well as $\ln(cn + n) + 1 + c/n = (\ln n) + O(1)$. Note that $(\mu + \lambda)n \log \mu = o(n \log n)$ by conditions on~$\mu, \lambda$, hence this and all other small-order terms are absorbed in the term $o(1)$.

In order to prove the general bound~\eqref{eq:upper-bound-GAs-general-p}, we consider canonical fitness levels, i.\,e., the $i$-th fitness level contains all search points with fitness~$i$. We estimate the time spent on each level~$i$, i.\,e., when the best fitness in the current population is~$i$. For each fitness level we consider three cases. The first case applies when the population contains individuals on fitness levels less than~$i$. The second case is when the population only contains copies of a single individual on level~$i$. The third case occurs when the population contains more than one individual on level~$i$; then the population contains different ``building blocks'' that can be recombined effectively by crossover.

All these cases capture the typical behaviour of a \mlGA, albeit some of these cases, and even whole fitness levels, may be skipped. We obtain an upper bound on its expected optimization time by summing up expected times the \mlGA may spend in all cases and on all fitness levels.

\textbf{Case $i.1$:} The population contains an individual on level~$i$ and at least one individual on a lower fitness level.

A sufficient condition for leaving this case is that all individuals in the population obtain fitness at least~$i$. Since the \mlGA never accepts worsenings, the case is left for good.

The time for all individuals reaching fitness at least~$i$ has already been estimated in Lemma~\ref{lem:takeover}. Applying this lemma to all fitness levels~$i$, the overall time spent in all cases~$i.1$ is at most
\[
\frac{O((\mu+\lambda)n \log \mu)}{(1-p_c)(1-p)^{n}} =
\frac{O((\mu+\lambda)n \log \mu)}{(1-p)^{n}}.
\]

\textbf{Case $i.2$:} The population contains $\mu$ copies of the same individual~$x$ on level~$i$.

In this case, each offspring created by the \mlGA will be a standard mutation of~$x$. This is obvious for offspring where the \mlGA decides not to use crossover. If crossover is used, the \mlGA will pick $x_1, x_2 = x$, create $y=x$ by crossover, and hence perform a mutation on~$x$.

The \mlGA leaves this case for good if either a better search point is created or if it creates another search point with~$i$ ones. In the latter case we will create a population with two different individuals on level~$i$. Note that due to the choice of the tie-breaking rule in the environmental selection, the \mlGA will always maintain at least two individuals on level~$i$, unless an improvement with larger fitness is found.

The probability of creating a better search point in one mutation is at least $(n-i) \cdot {p (1-p)^{n-1}}$ as there are $n-i$ suitable 1-bit flips. The probability of creating a different search point on level~$i$ is at least
$i(n-i) \cdot p^2 (1-p)^{n-2}$ as it is sufficient to flip one of $i$ 1\nobreakdash-bits, to flip one of $n-i$ 0\nobreakdash-bits, and not to flip any other bit. The probability of either event happening in one offspring creation is thus at least
\begin{align*}
& (n-i) \cdot p (1-p)^{n-1} + i(n-i) \cdot p^2 (1-p)^{n-2}\\
\ge\;& p (1-p)^{n-1} \cdot (n-i)(1+ip).
\end{align*}
By Lemma~\ref{lem:success-probability-lambda-offspring}, the expected number of function evaluations in Case~$i.2$ is at most
\[
\lambda + \frac{1}{p (1-p)^{n-1} \cdot (n-i)(1+ip)}.
\]
The expected number of functions evaluations made in all cases~$i.2$ is hence at most
\begin{align}
& \lambda n + \sum_{i=0}^{n-1} \frac{1}{p (1-p)^{n-1} \cdot (n-i)(1+ip)}\notag\\
=\;& \lambda n + \frac{1}{p(1-p)^{n-1}} \cdot \sum_{i=0}^{n-1} \frac{1}{(n-i)(1+ip)}\label{eq:upper-bound-GA-onemax-p}.
\end{align}
The last sum can be estimated as follows. Separating the summand for $i=n-1$, \begin{align*}
& \sum_{i=0}^{n-2} \frac{1}{(n-i)(1+ip)} + \frac{1}{1+(n-1)p}\\
\le\;& \int_{i=0}^{n-1} \frac{1}{(n-i)(1+ip)} \;\mathrm{d}i + \frac{1 + p}{1+np}.
\end{align*}
We use equation~3.3.20 in~\cite{AbramovitzStegun10} to simplify the integral and get
\begin{align*}
& \left[\frac{1}{1+np} \cdot \ln\left(\frac{1+ip}{n-i}\right)\right]_0^{n-1} + \frac{1 + p}{1+np}\\
=\;& \frac{\ln(np+1-p)+\ln(n)}{1+np} + \frac{1 + p}{1+np}\\
\le\;& \frac{\ln(n^2 p + n) + 1+p}{1+np}.
\end{align*}
Plugging this into~\eqref{eq:upper-bound-GA-onemax-p} yields that the expected time in all cases~$i.2$ is at most
\[
\lambda n + \frac{\ln(n^2 p + n) + 1+p}{p(1-p)^{n-1} \cdot (1+np)}.
\]

\textbf{Case $i.3$:} The population only contains individuals on level~$i$, not all of which are identical.

In this case we can rely on crossover recombining two different individuals on level~$i$. As they both have different ``building blocks'', i.\,e., different bits are set to~1, there is a good chance that crossover will generate an offspring with a higher number of 1\nobreakdash-bits.

The probability of performing a crossover with two different parents in one offspring creation is at least
\[
p_c \cdot \frac{\mu-1}{\mu^2}
\]
as in the worst case the population contains $\mu-1$ copies of one particular individual.

Assuming two different parents are selected for crossover, let these have Hamming distance~$2d$ and let $X$ denote the number of 1\nobreakdash-bits among these positions in the offspring.
Note that $X$ is binomially distributed with parameters $2d$ and~$1/2$ and its expectation is~$d$. We estimate the probability of getting a surplus of 1\nobreakdash-bits as this leads to an improvement in fitness. This estimate holds for any $d \in \N$.
Since $\Prob(X < d) = \Prob(X > d)$, we have
\[
\Prob(X > d) = \frac{1}{2} \left(1 - \Prob(X = d)\right)
= \frac{1}{2} \left(1 - 2^{-2d} \binom{2d}{d}\right) \ge \frac{1}{4}.
\]
Mutation keeps all 1\nobreakdash-bits with probability at least $(1-p)^{n}$. Together, the probability of increasing the current best fitness in one offspring creation is at least
\[
p_c \cdot \frac{\mu-1}{\mu^2} \cdot \frac{(1-p)^{n}}{4}.
\]
By Lemma~\ref{lem:success-probability-lambda-offspring}, the expected number of function evaluations in Case~$i.3$ is at most
\[
\lambda + \frac{4\mu^2}{p_c \cdot (\mu-1) \cdot (1-p)^n}.
\]
The total expected time spent in all cases~$i.3$ is hence at most
\[
\lambda n + \frac{4\mu^2  n}{p_c \cdot (\mu-1) \cdot (1-p)^{n}}\\
= \lambda n + \frac{O(\mu  n)}{(1-p)^{n}}
\]
as $p_c = \Omega(1)$.

Summing up all expected times yields a total time bound of
\begin{align*}
& \frac{\ln(n^2 p + n) + 1+p}{p(1-p)^{n-1} \cdot (1+np)} + 2\lambda n +
\frac{O(\mu n) + O((\mu+\lambda)n \log \mu)}{(1-p)^{n}}\\
=\;&
\frac{\ln(n^2 p + n) + 1+p}{p(1-p)^{n-1} \cdot (1+np)} +
\frac{O((\mu+\lambda)n \log \mu)}{(1-p)^{n}}.\qedhere
\end{align*}
\end{proof}

\begin{remark}[On conditions for $\mu$ and $\lambda$]
The second statement of Theorem~\ref{the:upper-bound-uniform-crossover-variable-p-every-GA} requires $\mu, \lambda = o((\log n)/(\log \log n))$ in order to establish the upper bound in~\eqref{eq:upper-bound-GAs-one-plus-o-one}. This condition seems necessary as for larger values of $\mu$ and $\lambda$ the inertia of a large population slows down exploitation, at least in the absence of crossover. Note not all EAs covered by Theorem~\ref{the:upper-bound-uniform-crossover-variable-p-every-GA} (after removing crossover) optimize \ONEMAX in time~$O(n \log n)$.

Witt~\cite{Witt2006} showed that a \muea with uniform parent selection has an expected optimization time of~$\Omega(\mu n + n \log n)$ on \ONEMAX. For $\mu = \omega(\log n)$, this lower bound is $\omega(n \log n)$. Jansen, De~Jong, and Wegener~\cite{Jansen2005a} showed that a (1+$\lambda$)~EA needs time $\omega(n \log n)$ on \ONEMAX if $\lambda = \omega((\log n)(\log \log n)/(\log \log \log n))$. Badkobeh, Lehre, and Sudholt~\cite{Badkobeh2014} showed that every black-box algorithm creating $\lambda$ offspring, using only standard bit mutation or other unary unbiased operators, needs time $\omega(n \log n)$ on \ONEMAX for $\lambda = \omega((\log n)(\log \log n))$. This indicates that the threshold in the condition $\mu, \lambda = o((\log n)/(\log \log n))$ is tight up to polynomials of $\log \log n$.
\end{remark}

\begin{remark}[On conditions for~$p_c$]
\label{rem:on-conditions-for-pc}
Theorem~\ref{the:upper-bound-uniform-crossover-variable-p-every-GA} assumes $0 < p_c < 1$ constant, which reflects the most common choices in applications of EAs. The theorem can be extended towards smaller or larger values as follows. If $p_c = o(1)$ the upper bound on the time spent in Cases~$i.3$ increases as it contains a factor of $1/p_c$. The other cases remain unaffected, and if $((\mu + \lambda) \log \mu)/p_c = o(\log n)$ we still get the upper bound from~\eqref{eq:upper-bound-GAs-one-plus-o-one}.

For high crossover probabilities, that is, $p_c = 1 - o(1)$ or $p_c = 1$, only Cases~$i.1$ need to be revisited. The time in those cases was derived from Lemma~\ref{lem:takeover}, which can be adapted as follows: the probability for increasing the number of fit individuals is at least
\[
p_c \cdot (1-p)^n \cdot \frac{(\text{number of fit individuals in population})^2}{2\mu^2}
\]
as it suffices to select two fit individuals and generate an average or above-average number of 1\nobreakdash-bits in the offspring, which happens with probability at least $1/2$. The time bound from Lemma~\ref{lem:takeover} then becomes
\[
\frac{O(\mu^2 + \lambda \log \mu)}{(1-p)^n}
\]
and the time bound in Theorem~\ref{the:upper-bound-uniform-crossover-variable-p-every-GA} becomes
\[
\frac{\ln(n^2 p + n) + 1 + p}{p(1-p)^{n-1} \cdot (1+np)} + \frac{O(n(\mu^2 + \lambda \log \mu))}{(1-p)^n}.
\]
For $p=c/n$, $c > 0$ constant, and $\mu, \lambda = o(\sqrt{\log n})$, this also establishes the upper bound from~\eqref{eq:upper-bound-GAs-one-plus-o-one}.
\end{remark}

It is remarkable that the waiting time for successful crossovers in Cases~$i.3$ is only of order~$O((\mu + \lambda)n)$.
For small values of $\mu$ and $\lambda$, e.\,g.~$\mu, \lambda = O(1)$, the time spent in all Cases~$i.3$ is $O(n)$, which is negligible compared to the overall time bound of order $\Theta(n \log n)$. This shows how effective crossover is in recombining building blocks.

Also note that the proof of Theorem~\ref{the:upper-bound-uniform-crossover-variable-p-every-GA} is relatively simple as it uses only elementary arguments and, along with Lemmas~\ref{lem:takeover} and~\ref{lem:success-probability-lambda-offspring}, it is fully self-contained. The analysis therefore lends itself for teaching purposes on the behavior of evolutionary algorithms and the benefits of crossover.

Our analysis has revealed that fitness-neutral mutations, that is, mutations creating a different search point of the same fitness, can help to escape from the case of a population with identical individuals. Even though these mutations do not immediately yield an improvement in terms of fitness, they increase the diversity in the population. Crossover is very efficient in exploiting this gained diversity by combining two different search points at a later stage. This means that crossover can capitalize on mutations that have both beneficial and disruptive effects on building blocks.

An interesting consequence is that this affects the optimal mutation rate on \ONEMAX. For EAs using only standard bit mutations Witt~\cite{Witt2013} recently proved that $1/n$ is the optimal mutation rate for the \EA on all linear functions. Recall that the \EA is the optimal mutation-based EA (in the sense of Theorem~\ref{the:lower-onemax}) on \ONEMAX~\cite{Sudholt2012c}.

For mutation-based EAs on \ONEMAX neutral mutations are neither helpful nor detrimental. With crossover neutral mutations now become helpful. Increasing the mutation rate increases the likelihood of neutral mutations. In fact, we can easily derive better upper bounds from Theorem~\ref{the:upper-bound-uniform-crossover-variable-p-every-GA} for slightly larger mutation rates, thanks to the additional term $1+np$ in the denominator of the upper bound.

The dominant term in~\eqref{eq:upper-bound-GAs-one-plus-o-one},
\[
\frac{n \ln n}{c \cdot e^{-c} \cdot (1+c)}
\]
is minimized for $c$ being the golden ratio $c = (\sqrt{5}+1)/2 \approx 1.618$.
This leads to the following.
\begin{corollary}
\label{the:best-runtime}
The asymptotically best running time bound from Theorem~\ref{the:upper-bound-uniform-crossover-variable-p-every-GA} is obtained for $p = (1+\sqrt{5})/(2n)$. For this choice the dominant term in~\eqref{eq:upper-bound-GAs-one-plus-o-one} becomes
\[
\frac{e^{(\sqrt{5}+1)/2}}{\sqrt{5}+2} \cdot n \ln n
\approx 1.19 n \ln n.
\]
\end{corollary}

\section{The Optimal Mutation Rate}
\label{sec:optimal-mutation-rate}

Corollary~\ref{the:best-runtime} gives the mutation rate that yields the best upper bound on the running time that can be obtained with the proof of Theorem~\ref{the:upper-bound-uniform-crossover-variable-p-every-GA}. However, it does not establish that this mutation rate is indeed optimal for any GA. After all, we cannot exclude that another mutation rate leads to a smaller expected optimization time.

In the following, we show for a simple (2+1)~GA (Algorithm~\ref{alg:TwoGA}) that the upper bound from Theorem~\ref{the:upper-bound-uniform-crossover-variable-p-every-GA} is indeed tight up to small-order terms, which establishes $p=(1+\sqrt{5})/(2n)$ as the optimal mutation rate for that (2+1)~GA. Proving lower bounds on expected optimization times is often a notoriously hard task, hence we restrict ourselves to a simple ``bare-bones'' GA that captures the characteristics of GAs covered by Theorem~\ref{the:upper-bound-uniform-crossover-variable-p-every-GA} and is easy to analyze. The latter is achieved by fixing as many parameters as possible.

As the upper bound from Theorem~\ref{the:upper-bound-uniform-crossover-variable-p-every-GA} grows with $\mu$ and $\lambda$, we pick the smallest possible values: $\mu=2$ and $\lambda=1$. The parent selection is made as simple as possible: we select parents uniformly at random from the current best individuals in the population. In other words, if we define the parent population as the set of individuals that have a positive probability to be chosen as parents, the parent population only contains individuals of the current best fitness. We call this parent selection ``greedy'' because it is a greedy strategy to choose the current best search points as parents.

In the context of the proof of Theorem~\ref{the:upper-bound-uniform-crossover-variable-p-every-GA} greedy parent selection implies that Cases~$i.1$ are never reached as the parent population never spans more than one fitness level. So the time spent in these cases is 0. This also allows us to eliminate one further parameter by setting $p_c=1$, as lower values for $p_c$ were only beneficial in Cases~$i.1$. Setting $p_c=1$ minimizes our estimate for the time spent in Cases~$i.3$.
So Theorem~\ref{the:upper-bound-uniform-crossover-variable-p-every-GA} extends towards this GA (see also Remark~\ref{rem:on-conditions-for-pc}).

We call the resulting GA ``greedy \TwoGA'' because its main characteristics is the greedy parent selection. The greedy \TwoGA is defined in Algorithm~\ref{alg:TwoGA}\footnote{Note that in~\cite{Sudholt2012b} the greedy~\TwoGA was defined slightly differently as there duplicate genotypes are always rejected. Algorithm~\ref{alg:TwoGA} is equivalent to the greedy~\TwoGA from~\cite{Sudholt2012b} for the following reasons. If the current population contains two different individuals of equal fitness and a duplicate of one of the parents is created, both algorithms reject a duplicate genotype. If the population contains two individuals of different fitness, both behave like the population only contained the fitter individual.}.


\begin{algorithm}
\label{alg:TwoGA}
Initialize population $\PCal$ of size $2$ u.\,a.\,r.\\
\While{true}{%
            Select $x_1, x_2$ u.\,a.\,r.\ from $\{x \in \PCal \mid \forall y \in \PCal \colon f(x) \ge f(y)\}$.\\
            Let $y := \text{crossover}(x_1, x_2)$.\\
		Flip each bit in $y$ independently with probability~$p$.\\
    Let $\PCal$ contain the $2$ best individuals from $\PCal \cup \{y\}$; break ties towards including individuals with the fewest duplicates in $\PCal \cup \{y\}$.
}
\caption{Greedy \TwoGA with mutation rate~$p$ for maximizing $f \colon \{0, 1\}^n \to \R$.}
\end{algorithm}


The following result applies to the greedy \TwoGA using any kind of mask-based crossover. A mask-based crossover is a recombination operator where each bit value is taken from either parent; that is, it is not possible to introduce a bit value which is not represented in any parent. All common crossovers are mask-based crossovers: uniform crossover, including parameterized uniform crossover, as well as $k$\nobreakdash-point crossovers for any~$k$. The following result even includes biased operators like a bit-wise OR, which induces a tendency to increase the number of 1\nobreakdash-bits.
\begin{theorem}
\label{the:lower-bound-GA-onemax-p}
Consider the greedy~\TwoGA with mutation rate~$0 < p \le 1/(\sqrt{n} \log n)$ using an arbitrary mask-based crossover operator.
Its expected optimization time on \ONEMAX is at least
\[
\frac{\min\{\ln n, \ln(1/(p^2 n))\} - O(\log \log n)}{(1+\max_k\{\frac{(pn)^k}{k!k!}\}) \cdot p(1-p)^{n}}.
\]
\end{theorem}
Before giving the proof, note that for $p = c/n$ with $0 < c \le 4$ constant, ${\max_k\{\frac{(pn)^k}{k!k!}\} = pn}$ as  for $0 < pn \le 4$ and $i \in \N$
$\frac{(pn)^{i+1}}{(i+1)!(i+1)!} = \frac{pn}{(i+1)^2} \cdot \frac{(pn)^i}{i!i!} \le \frac{(pn)^i}{i!i!}$, hence a maximum is attained for $k=1$. Then the lower bound from Theorem~\ref{the:lower-bound-GA-onemax-p} is
\[
\frac{n \ln n }{c \cdot e^{-c} \cdot (1+c)} - O(n \log \log n).
\]
This matches the upper bound~\eqref{eq:upper-bound-GAs-one-plus-o-one} up to small order terms, showing for the greedy \TwoGA that the new term $1+c$ in the denominator of the bound from Theorem~\ref{the:upper-bound-uniform-crossover-variable-p-every-GA} was not a coincidence.
For $p > 4/n$, the lower bound is at least
\[
(e+\Omega(1)) \cdot n \ln n.
\]
Together, this establishes the optimal mutation rate for the greedy~\TwoGA on \ONEMAX.
\begin{theorem}
\label{the:optimal-mutation-rate}
For the greedy~\TwoGA with uniform crossover on \ONEMAX mutation rate $p=(1+\sqrt{5})/(2n)$ minimizes the expected number of function evaluations, up to small-order terms.
\end{theorem}

For the proof of Theorem~\ref{the:lower-bound-GA-onemax-p} we use the following lower-bound technique based on fitness levels by the author~\cite{Sudholt2012c}.
\begin{theorem}[Sudholt~\cite{Sudholt2012c}]
\label{the:lower-bound-method}
Consider a partition of the search space into non-empty sets
$A_1, \dots, A_m$.
For a search algorithm~$\mathcal{A}$ we say that it is in $A_i$ or on level $i$ if the best individual created so far is in~$A_i$.
If there are $\chi, u_i, \gamma_{i, j}$ for $i < j$ where
\begin{enumerate}
\item the probability of traversing from level~$i$ to level~$j$ in one step is at most $u_i \gamma_{i, j}$ for all $i < j$,
\item $\sum_{j=i+1}^{m} \gamma_{i, j} = 1$ for all $i$, and
\item $\gamma_{i, j} \ge \chi \sum_{k=j}^{m} \gamma_{i, k}$ for all $i < j$ and some $0 \le \chi \le 1$,
\end{enumerate}
then the expected hitting time of $A_m$ is at least
\begin{align}
& \sum_{i=1}^{m-1} \Prob(\text{$\mathcal{A}$ starts in $A_i$}) \cdot \chi \sum_{j=i}^{m-1} \frac{1}{u_j}.\label{eq:simple-lower-bound}
\end{align}
\end{theorem}

\begin{proof}[Proof of Theorem~\ref{the:lower-bound-GA-onemax-p}]
We prove a lower bound for the following sped-up GA instead of the original greedy~\TwoGA. Whenever it creates a new offspring with the same fitness, but a different bit string as the current best individual, we assume the following. The algorithm automatically performs a crossover between the two. Also, we assume that this crossover leads to the best possible offspring in a sense that all bits where both parents differ are set to~1 (i.\,e., the algorithm performs a bit-wise OR). That is, if both search points have $i$ 1\nobreakdash-bits and Hamming distance $2k$, then the resulting offspring has $i+k$ 1\nobreakdash-bits.

Due to our assumptions, at the end of each generation there is always a single best individual.
For this reason we can model the algorithm by a Markov chain representing the current best fitness.


The analysis follows a lower bound for EAs on \ONEMAX~\cite[Theorem~9]{Sudholt2012c}.
As in~\cite{Sudholt2012c} we consider the following fitness-level partition that focuses only on the very last fitness values. Let $\ell = \ceil{n-\min\{n/\!\log n, 1/(p^2 n\log n)\}}$. Let $A_i = \{x \mid \ones{x} = i\}$ for $i > \ell$ and $A_{\ell}$ contain all remaining search points.
We know from~\cite{Sudholt2012c} that the GA is initialized in $A_{\ell}$ with probability at least $1-1/\log n$ if $n$ is large enough.

The probability $p_{i, i+k}$ that the sped-up GA makes a transition from fitness~$i$ to fitness~$i+k$ equals
\begin{align*}
p_{i, i+k} =\;& \Prob(\text{$k$ more 0\nobreakdash-bits than 1\nobreakdash-bits flip}) \;+ \\
& \Prob(\text{$k$ 0\nobreakdash-bits and $k$ 1\nobreakdash-bits flip})
\end{align*}
According to~\cite[Lemma~2]{Sudholt2012c}, for the considered fitness levels~$i > \ell$ the former probability is bounded by
\[
p^k (1-p)^{n-k} \cdot \frac{(n-i)^k}{k!} \cdot \left(1 + \frac{3}{5} \cdot \frac{i(n-i)p^2}{(1-p)^2}\right).
\]
%
The latter probability is bounded by
\begin{align*}
& \Prob(\text{$k$ 0\nobreakdash-bits flip}) \cdot \Prob(\text{$k$ 1\nobreakdash-bits flip})\\
\le\;& \frac{(n-i)^k}{k!} \cdot p^k(1-p)^{n-i-k} \cdot \frac{i^k}{k!} \cdot p^k(1-p)^{i-k}\\
\le\;& \frac{(n-i)^k}{k!} \cdot p^k(1-p)^{n} \cdot \frac{(pn)^k}{(1-p)^{2k} \cdot k!}.
\end{align*}

Together, $p_{i, i+k}$ is at most
\begin{align*}
& (p(n-i))^{k} (1-p)^{n} \left(1 + \frac{3}{5} \cdot \frac{i(n-i)p^2}{(1-p)^{2+k}} + \frac{(pn)^k}{(1-p)^{2k} \cdot k!k!}\right).
\end{align*}

We need to find variables $u_i$ and $\gamma_{i, i+k}$ along with some $0 \le \chi \le 1$ such that all conditions of Theorem~\ref{the:lower-bound-method} are fulfilled.
%
Define
\[
u_i' :=
p (1-p)^{n}  (n-i)
   \left(1 + \frac{3}{5} \cdot \frac{i(n-i)p^2}{(1-p)^3} + \frac{1}{(1-p)^{2}} \cdot \maxk\right)
\]
and
\[
\gamma_{i,i+k}' := \left(\frac{p(n-i)}{(1-p)^2}\right)^{k-1}.
\]
Observe that, for every $k \in \N$,
\begin{align*}
u_i' \gamma_{i, i+k}'
\ge\;& p^k (1-p)^{n}  (n-i)^k  \left(1 + \frac{3}{5} \cdot \frac{i(n-i)p^2}{(1-p)^{1+2k}} + \frac{1}{(1-p)^{2k}} \cdot \frac{(pn)^k}{k!k!}\right)\\
\ge\;& p^k (1-p)^{n}  (n-i)^k  \left(1 + \frac{3}{5} \cdot \frac{i(n-i)p^2}{(1-p)^{2+k}} + \frac{(pn)^k}{(1-p)^{2k} \cdot k!k!}\right)\\
\ge\;& p_{i, i+k}.
\end{align*}
In order to fulfill the second condition in Theorem~\ref{the:lower-bound-method}, we consider the following normalized variables:
$u_i := u_i' \cdot \sum_{j=i+1}^{n} \gamma_{i, j}'$ and $\gamma_{i, j} := \frac{\gamma_{i, j}'}{\sum_{j=i+1}^{n} \gamma_{i, j}'}$.
As $u_i \gamma_{i, j} = u_i' \gamma_{i, j}' \ge p_{i, j}$, this proves the first condition of Theorem~\ref{the:lower-bound-method}.

Following the proof of Theorem~9 in~\cite{Sudholt2012c}, it is easy to show that for ${\chi := 1-\frac{1}{(1-p)^2 \log n}}$ we get $\gamma_{i, j} \ge \chi \sum_{k=j}^{m} \gamma_{i, k}$ for all $i, j$ with $j > i$ (the calculations on~\cite[pp.\ 427--428]{Sudholt2012c} carry over by replacing ``$(1-p)$'' with ``$(1-p)^2$''). This establishes the third and last condition.

As $\gamma_{i, j} \ge \chi \sum_{k=j}^{m} \gamma_{i, k}$ is equivalent to
$\gamma_{i, j}' \ge \chi \sum_{k=j}^{m} \gamma_{i, k}'$, we get
\[
\sum_{j=i+1}^n \gamma_{i, j}' \le \frac{\gamma_{i, i+1}'}{\chi} \le \frac{1}{\chi},
\]
which implies, using $i(n-i)p^2 \le n(n-\ell)p^2 \le \frac{1}{\log n}$~\cite[(12)]{Sudholt2012c} as well as $1+x \le 1/(1-x)$ for $x < 1$,
\begin{align*}
u_i \le\; &
p (1-p)^{n} \cdot (n-i) \cdot \frac{1}{\chi}
\cdot \left(1 + \frac{3}{5} \cdot \frac{i(n-i)p^2}{(1-p)^3} + \frac{1}{(1-p)^{2}} \cdot \max_k\left(\frac{(pn)^k}{k!k!}\right)\right)\\
\le\; &
p (1-p)^{n-3} \cdot (n-i) \cdot \frac{1}{\chi}
\cdot \left(1 + \frac{3}{5\log n} + \max_k\left(\frac{(pn)^k}{k!k!}\right)\right)\\
\le\; &
p (1-p)^{n-3} \cdot (n-i) \cdot \frac{1}{\chi} \cdot \left(\frac{1}{1- \frac{3}{5\log n}} + \max_k\left(\frac{(pn)^k}{k!k!}\right)\right)\\
\le\; &
p (1-p)^{n-3} \cdot (n-i) \cdot \frac{1}{\chi} \cdot \frac{1 + \max_k\left(\frac{(pn)^k}{k!k!}\right)}{1- \frac{3}{5\log n}}.
\end{align*}

Invoking Theorem~\ref{the:lower-bound-method} and recalling that the first fitness level is reached with probability at least $1-1/\log n$, we get a lower bound of
\begin{align*}
& \left(1 - \frac{1}{\log n}\right) \chi \sum_{i=\ell}^{n-1} \frac{1}{u_i}\\
\ge\;&
\left(1 - \frac{1}{\log n}\right) \chi^2
\cdot
\frac{1- \frac{3}{5 \log n}}
{1 + \max_k\left(\frac{(pn)^k}{k!k!}\right)}
\cdot \frac{(1-p)^3}{p(1-p)^{n}} \sum_{i=\ell}^{n-1} \frac{1}{n-i}\\
\ge\;&
\left(1 - \mathord{O}\mathord{\left(\frac{1}{\log n}\right)}\right) \cdot \frac{1}{1+\maxk} \cdot \frac{1}{p(1-p)^n} \sum_{i=\ell}^{n-1} \frac{1}{n-i}
\end{align*}
where in the last step we used that all factors $\chi,  1- \frac{3}{5 \log n}$, and $1 - p$ are $1 - \mathord{O}\mathord{\left(\frac{1}{\log n}\right)}$, and $\left(1-\frac{c}{\log n}\right)^{d} \ge 1 - \frac{cd}{\log n}$ for any positive constants $c, d$.
Bounding $\sum_{i=\ell}^{n-1} \frac{1}{n-i} \le \ln(\min\{n, 1/(p^2n)\}) - \ln(\log n)$ as in~\cite{Sudholt2012c} and absorbing all small-order terms in the $-O(\log \log n)$ term from the statement gives the claimed bound.
\end{proof}

\pgfplotsset{every axis legend/.append style={
at={(0.5,1)},
anchor=north}}

We also ran experiments to see whether the outcome matches our inspection of the dominating terms in the running time bounds for realistic problem dimensions. We chose $n=1000$ bits and recorded the average optimization time over 1000 runs. The mutation rate~$p$ was set to~$c/n$ with $c \in \{0.1, 0.2, \dots, 4\}$. The result is shown in Figure~\ref{fig:runtime-uniform}.

\begin{figure}[hbt]
\centerline{
\begin{tikzpicture}[scale=0.9]
\begin{axis}[legend cell align=left, xlabel={mutation rate},ylabel={number of evaluations},
xtick scale label code/.code={$\cdot 1/n$},
extra x ticks={0.001,0.001618033989},
extra x tick style={grid=major,xticklabel shift=1},
extra x tick labels={, $\frac{\sqrt{5}+1}{2}$},
]
\pgfplotstableread[header=false]{experiments-onemax/(1+1)EA-onemax-1000runs-n=1000-p=0.1-by-n,...,4-by-n.txt}\pOneOne
\pgfplotstableread[header=false]{experiments-onemax/greedy-GA-fastxover-onemax-1000runs-n=1000-p=0.1-by-n,...,4-by-n.txt}\pGreedyGA
\pgfplotstableread[header=false]{experiments-onemax/stat-greedy-GA-1ptxover-onemax-n=1000-p=0.1-by-n,...,4-by-n.txt}\pGAonept
\pgfplotstableread[header=false]{experiments-onemax/stat-greedy-GA-2ptxover-onemax-n=1000-p=0.1-by-n,...,4-by-n.txt}\pGAtwopt
\addplot table[x expr=0.1*0.001*(\thisrow{1}), y index=7]{\pOneOne};
\addplot table[x expr=0.1*0.001*(\thisrow{1}), y index=7]{\pGreedyGA};
\pgfplotsset{cycle list shift=-2}
\addplot[mark=none,blue] table[x expr=0.1*0.001*(\thisrow{1}), y expr=\thisrow{7} + \thisrow{11}]{\pOneOne};
\pgfplotsset{cycle list shift=-1}
\addplot[mark=none,blue] table[x expr=0.1*0.001*(\thisrow{1}), y expr=\thisrow{7} - \thisrow{11}]{\pOneOne};
\addplot[mark=none,red] table[x expr=0.1*0.001*(\thisrow{1}), y expr=\thisrow{7} + \thisrow{11}]{\pGreedyGA};
\pgfplotsset{cycle list shift=-1}
\addplot[mark=none,red] table[x expr=0.1*0.001*(\thisrow{1}), y expr=\thisrow{7} - \thisrow{11}]{\pGreedyGA};
\legend{\EA, \TwoGA{}+uniform}
\end{axis}
\end{tikzpicture}
}
\caption{Average optimization times for the \EA and the greedy~\TwoGA with uniform crossover on OneMax with $n=1000$ bits. The mutation rate~$p$ is set to~$c/n$ with $c \in \{0.1, 0.2, \dots, 4\}$. The thin lines show mean $\pm$ standard deviation.}
\label{fig:runtime-uniform}
\end{figure}
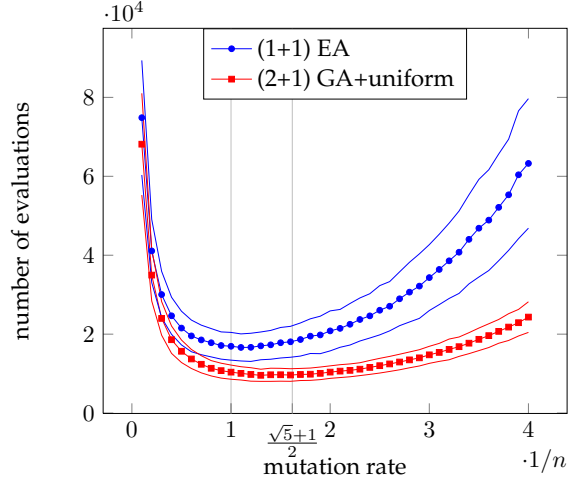

One can see that for every mutation rate the greedy~\TwoGA has a lower average optimization time. As predicted, the performance difference becomes larger as the mutation rate increases. The optimal mutation rates for both algorithms match with minimal average optimization times. Note that also the variance/standard deviation was much lower for the GA for higher mutation rates.
Preliminary runs for $n=100$ and $n=10000$ bits gave very similar results. 
More experiments and statistical tests are given in Section~\ref{sec:experiments}.

\section{\boldmath{\large$k$}-Point Crossover}
\label{sec:k-point-crossover}

The $k$\nobreakdash-point crossover operator picks $k$ cutting points from $\{1, \dots, n-1\}$ uniformly at random without replacement. These cutting points divide both parents into segments that are then assembled from alternating parents. That is, for parents $x, y$ and cutting points $1 \le \ell_1 < \ell_2 < \dots < \ell_k \le n-1$ the offspring will be:
\[
x_{1} \dots x_{\ell_1} \ y_{\ell_1+1} \dots y_{\ell_2} \ x_{\ell_2+1} \dots x_{\ell_3} \ y_{\ell_3+1} \dots y_{\ell_4} \quad \dots
\]
the suffix being $y_{\ell_k+1} \dots y_n$ if $k$ is odd and $x_{\ell_k+1} \dots x_n$ if $k$ is even.

For uniform crossover we have seen that populations containing different search points of equal fitness are beneficial as uniform crossover can easily combine the good ``building blocks''. This holds regardless of the Hamming distance between these different individuals, and the position of bits where individuals differ.

The \mlGA with $k$\nobreakdash-point crossover is harder to analyse as there the probability of crossover creating an improvement depends on the Hamming distance of parents and the position of differing bits.

Consider parents that differ in two bits, where these bit positions are quite close. Then $1$-point crossover has a high probability of taking both bits from the same parent. In order to recombine those building blocks, the cutting point has to be chosen in between the two bit positions. A similar effect occurs for 2-point crossover also if the two bit positions are on opposite ends of the bit string.

The following lemma gives a lower bound on the probability that $k$\nobreakdash-point crossover combines the right building blocks on \ONEMAX, if two parents are equally fit and differ in two bits. The lemma and its proof may be of independent interest.
\begin{lemma}
\label{lem:crossover-successful}
Consider two search points $x, y$ with ${x_i = 1}, {x_{i+d} = 0}, {y_i = 0}, {y_{i+d} = 1}$ for $1 \le i < i+d \le n$ and $x_s = y_s$ for $s \notin \{i, i+d\}$.
The probability of $k$\nobreakdash-point crossover of $x$ and~$y$, for any $1 \le k \le N-1$, where $N := n-1 \ge 4$ is the number of possible cutting points, creating an offspring with a larger number of 1\nobreakdash-bits is at least
\[
\frac{d(N-d)}{N(N-1)}
\]
and exactly $d/N$ for $k=1$.
\end{lemma}
\begin{proof}
%
We identify cutting points with bits such that cutting point~$a$ results in two strings $x_1 \dots x_a$ and $x_{a+1} \dots x_n$.
We say that a cutting point~$a$ separates $i$ and $i+d$ if $a \in \{i, \dots, i+d-1\}$. Note that the prefix is always taken from~$x$. The claim now follows from showing that the number of separating cutting points is odd with the claimed probability.

Let $X_{N, d, k}$ be the random variables that describes the number of cutting points separating $i$ and~$i+d$. This variable follows a hypergeometric distribution $\Hyp{N}{d}{k}$, illustrated by the following urn model with red and white balls. The urn contains $N$ balls, $d$ of which are red. We draw $k$ balls uniformly at random, without replacement. Then $X_{N, d, k}$ describes the number of red balls drawn. We define the probability of $X_{N, d, k}$ being odd, for $1 \le d \le N-1$ and $1 \le k \le N-1$ as
\[
P(N, d, k) := \sum_{x=1, \ x \ \mathrm{odd}}^k \Prob(X_{N, d, k} = x) = \sum_{x=1, \ x \ \mathrm{odd}}^k \frac{\binom{d}{x}\binom{N-d}{k-x}}{\binom{N}{k}}.
\]
Note that for $k=1$
\[
P(N, d, 1) = \frac{\binom{d}{1}\binom{N-d}{0}}{\binom{N}{1}} = \frac{d}{N}
\]
and for $k=2$
\[
P(N, d, 2) = \frac{\binom{d}{1}\binom{N-d}{1}}{\binom{N}{2}} = \frac{2d(N-d)}{N(N-1)}.
\]
For all $1 \le d \le N-1$ and all~$1 \le k \le N-1$ the following recurrence holds. Imagine drawing the first cutting point separately. With probability $d/N$, the cutting point is a separating cutting point, and then we need an even number of further separating cutting points among the remaining $k-1$ cutting points, drawn from a random variable $X_{N-1, d-1, k-1}$. With the remaining probability $(N-d)/N$, the number of remaining cutting points must be even, and this number is drawn from a random variable $X_{N-1, d, k-1}$. Hence
\begin{align}
\label{eq:k-point-recurrence}
P(N, d, k) =\;& \frac{d}{N} \cdot (1 - P(N-1, d-1, k-1)) + \frac{N-d}{N} \cdot P(N-1, d, k-1).
\end{align}
Assume for an induction that for all $2 \le k' < k$
\begin{equation}
\label{eq:k-point-induction-assumption}
\frac{d(N-d)}{N(N-1)} \le P(N, d, k') \le 1 - \frac{d(N-d)}{N(N-1)}.
\end{equation}
This is true for $k'=2$ as, using $3d(N-d) \le 3 \cdot (N/2)^2 \le N(N-1)$ for $N \ge 4$,
\[
P(N, d, 2) = \frac{2d(N-d)}{N(N-1)} = \frac{3d(N-d) - d(N-d)}{N(N-1)} \le 1 - \frac{d(N-d)}{N(N-1)}.
\]
For $k > 2$, combining~\eqref{eq:k-point-recurrence} and~\eqref{eq:k-point-induction-assumption} yields
\begin{align*}
P(N, d, k) =\;& \frac{d}{N} \cdot (1 - P(N-1, d-1, k-1)) + \frac{N-d}{N} \cdot P(N-1, d, k-1)\\
\ge\;& \frac{d}{N} \cdot \frac{(d-1)(N-d)}{(N-1)(N-2)} + \frac{N-d}{N} \cdot \frac{d(N-d-1)}{(N-1)(N-2)}\\
=\;& \frac{d(N-d)(d-1 + N-d-1)}{N(N-1)(N-2)}\\
=\;& \frac{d(N-d)}{N(N-1)}.
\end{align*}
The upper bound follows similarly:
\begin{align*}
P(N, d, k)
\le\;& \frac{d}{N} \cdot \left(1 - \frac{(d-1)(N-d)}{(N-1)(N-2)}\right) + \frac{N-d}{N} \cdot \left(1 - \frac{d(N-d-1)}{(N-1)(N-2)}\right)\\
=\;& 1 - \frac{d(N-d)(d-1 + N-d-1)}{N(N-1)(N-2)}\\
=\;& 1 - \frac{d(N-d)}{N(N-1)}.
\end{align*}
By induction, the claim follows.
\end{proof}

In the setting of Lemma~\ref{lem:crossover-successful}, the probability of $k$\nobreakdash-point crossover creating an improvement depends on the distance between the two differing bits.
Fortunately, for search points that result from a mutation of one another, this distance has a favourable distribution. This is made precise in the following lemma.
\begin{lemma}
\label{lem:bit-distance-distribution}
Let $x'$ result from $x$ by a mutation flipping one 1-bit and one 0-bit, where the positions~$i, j$ of these bits are chosen uniformly among all 1\nobreakdash-bits and 0\nobreakdash-bits, respectively.
Then for $d := |i-j|$ the random variable $\min\{d, n-d\}$ stochastically dominates the uniform distribution on $\{1, \dots, n/4\}$.
\end{lemma}
\begin{proof}
We first show the following. For any fixed index~$i$ and any integer~$1 \le z < n/2$ there are exactly two positions~$j$ such that $\min\{d, n-d\} = z$. If $i \in \{1, \dots, n\}$ and $z \in \N$ are fixed, the only values for~$j$ that result in either $|i-j|=z$ or $n-|i-j|=z$ are $i+z, i-z, i+z-n$, and $i-z+n$. Note that at most two of these values are in $\{1, \dots, n\}$. Hence, there are at most 2 feasible values for~$j$ for every $d \in \N$.

Let $\ell$ denote the number of 1\nobreakdash-bits in~$x$. If $\ell \ge n/2$, we assume that first the 0\nobreakdash-bit is chosen uniformly at random, and then consider the uniform random choice of a corresponding 1\nobreakdash-bit.
Without loss of generality assume $x_i = 0$ and $x_j = 1$.

When $i$ has been fixed and $j$ is chosen uniformly at random, a worst case distribution for $\min\{|i-j|, n-|i-j|\}$ is attained when the 1\nobreakdash-bits are distributed such that for each $1 \le d \le \lfloor \ell/2\rfloor$ both feasible bit positions are~1. The worst case is hence a uniform distribution on $\{1, \dots, \lfloor \ell/2 \rfloor\}$, which stochastically dominates the uniform distribution on $\{1, \dots, n/4\}$.

The case $\ell < n/2$ is symmetrical: exchanging the roles of $x_j$ and $x_i$ as well as the roles of zeros and ones yields a uniform distribution on the set ${\{1, \dots, \lfloor (n-\ell)/2 \rfloor\}}$ as worst case, which again stochastically dominates the uniform distribution on $\{1, \dots, n/4\}$. \end{proof}

Taken together, Lemma~\ref{lem:crossover-successful} and Lemma~\ref{lem:bit-distance-distribution} indicate that $k$\nobreakdash-point crossover has a good chance of finding improvements through recombining the right ``building blocks''. However, this is based on the population containing potential parents of equal fitness that only differ in two bits.

The following analysis shows that the population is likely to contain such a favourable pair of parents. However, such a pair might get lost again if other individuals of the same fitness are being created, after all duplicates have been removed from the population. For parents that differ in more than 2 bits, Lemma~\ref{lem:crossover-successful} does not apply, hence we do not have an estimate of how likely such a crossover will find an improvement.

In order to avoid this problem, we consider a more detailed tie-breaking rule. As before, individuals with fewer duplicates are being preferred. In case there are still ties after considering the number of duplicates, the \mlGA will retain older individuals. This refined tie-breaking rule is shown in Algorithm~\ref{alg:refined-tie-breaking}. As will be shown in the remainder, it implies that once a favourable pair of parents with Hamming distance~2 has been created, this pair will never get lost.

\LinesNotNumbered
\begin{algorithm}
\Indp%
\nlset{14}%
    Let $\PCal$ contain the $\mu$ best individuals from $\PCal \cup \PCal'$; break ties towards including individuals with the fewest duplicates in $\PCal \cup \PCal'$. If there are still ties, break them towards including older individuals.
\caption{Refined tie-breaking rule ``dup-old''.}
\label{alg:refined-tie-breaking}
\end{algorithm}

\begin{figure}[tb]
\centerline{
\begin{tikzpicture}[scale=0.8]
\begin{axis}[legend cell align=left, xlabel={mutation rate},ylabel={number of evaluations},
xtick scale label code/.code={$\cdot 1/n$},
]
\pgfplotstableread[header=false]{experiments-onemax/greedy-GA-1ptxover-dup+oldtiebreak-onemax-1000runs-n=1000-p=0.1-by-n,...,4-by-n.txt}\pOldOnePt
\pgfplotstableread[header=false]{experiments-onemax/greedy-GA-2ptxover-dup+oldtiebreak-onemax-1000runs-n=1000-p=0.1-by-n,...,4-by-n.txt}\pOldTwoPt
\pgfplotstableread[header=false]{experiments-onemax/greedy-GA-1-point-crossover-onemax-1000runs-n=1000-p=0.1-by-n,...,4-by-n.txt}\pGAonept
\pgfplotstableread[header=false]{experiments-onemax/greedy-GA-2-point-crossover-onemax-1000runs-n=1000-p=0.1-by-n,...,4-by-n.txt}\pGAtwopt
\addplot table[x expr=0.1*0.001*(\thisrow{1}), y index=7]{\pGAonept};
\addplot table[x expr=0.1*0.001*(\thisrow{1}), y index=7]{\pGAtwopt};
\addplot table[x expr=0.1*0.001*(\thisrow{1}), y index=7]{\pOldOnePt};
\addplot table[x expr=0.1*0.001*(\thisrow{1}), y index=7]{\pOldTwoPt};
\legend{{\TwoGA+1-point, dup-rnd}, {\TwoGA+2-point, dup-rnd}, {\TwoGA+1-point, dup-old}, {\TwoGA+2-point, dup-old}}
\end{axis}
\end{tikzpicture}
}
\caption{Average optimization times on \ONEMAX with $n=1000$ bits over 1000 runs for the greedy~\TwoGA with $1$- and $2$-point crossover using different tie-breaking rules if individuals are tied with regard to fitness and the number of duplicates. ``dup-rnd'' breaks these ties randomly, whereas ``dup-old'' (Algorithm~\ref{alg:refined-tie-breaking}) prefers older individuals.
The mutation rate~$p$ is set to~$c/n$ with $c \in \{0.1, 0.2, \dots, 4\}$.}
\label{fig:runtime-tie-breaking}
\end{figure}
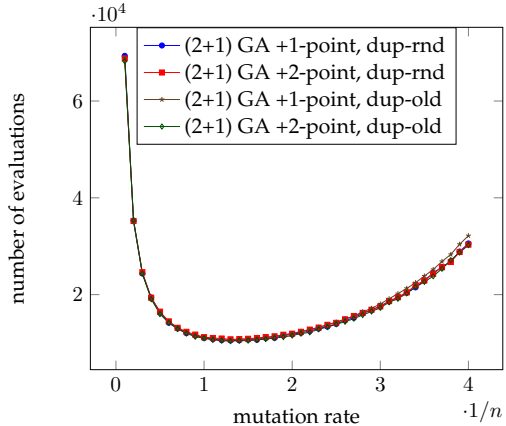

This tie-breaking rule, called ``dup-old'' differs from the one used for the experiments in Figure~\ref{fig:runtime-uniform} and those in Section~\ref{sec:extensions}. There, we broke ties uniformly at random in case individuals are tied with respect to both fitness and the number of duplicates. We call the latter rule ``dup-rnd''.
Experiments for the greedy \TwoGA comparing tie-breaking rules dup-old and dup-rnd over 1000 runs indicate that performance differences are very small, see Figure~\ref{fig:runtime-tie-breaking}.\footnote{Even though differences are small, one-sided Mann-Whitney $U$ tests reveal some statistically significant differences: for 1-point crossover dup-rnd is significantly faster than dup-old on a significance level of 0.001 for mutation rates at least $2.4/n$ (with two exceptions, $2.8/n$ and $3.6/n$, with $p$-values still below 0.003). Contrarily, dup-old was significantly faster for 2-point crossover for mutation rates in the range of $0.8/n$ to $3/n$.}

Note, however, that on functions with plateaus, like royal road functions, retaining the older individuals prevents the \mlGA from performing random walks on the plateau, once the population has spread such that there are no duplicates of any individual. In this case we expect that performance will deteriorate when breaking ties towards older individuals.

With the refined tie-breaking rule, the performance of \mlGA{}s is as follows.
\begin{theorem}
\label{the:upper-bound-k-point-crossover-variable-p}
The expected optimization time of every \mlGA implementing Algorithm~\ref{alg:Scheme-GA} with tie-breaking rule dup-old from Algorithm~\ref{alg:refined-tie-breaking},
$2 \le \mu = O(1)$, $\lambda < \mu$, $p_c = o(1)$ and $p_c = \omega(1/\log n)$, $p = c/n$ for some constant~$c > 0$, and $k$\nobreakdash-point crossover with any $1 \le k \le n-2$, on \ONEMAX is at most
\[
\frac{n \ln n}{c \cdot e^{-c} \cdot (1+c)} \cdot (1 + o(1)).
\]
\end{theorem}
This bound equals the upper bound~\eqref{eq:upper-bound-GAs-one-plus-o-one} for \mlGA{}s with uniform crossover. It improves upon the previous upper bound for the greedy \TwoGA from~\cite[Theorem~8]{Sudholt2012b}, whose dominant term was by an additive term of $\frac{2c}{3 + 3c} \cdot n \ln n$ larger. The reason was that for the \TwoGA favourable parents could get lost, which is now prevented by the dup-old tie-breaking rule and conditions on $p_c$.

The conditions $p_c = o(1)$ as well as $\mu, \lambda = O(1)$ are useful because they allow us to estimate the probability that a single good individual takes over the whole population with copies of itself.

In the remainder of this section we work towards proving Theorem~\ref{the:upper-bound-k-point-crossover-variable-p} and assume that $n \ge n_0$ for some $n_0$ chosen such that all asymptotic statements that require a large enough value of~$n$ hold true. For $n < n_0$ there is nothing to prove, as the statement holds trivially for bounded~$n$.

We again estimate the time spent on each fitness level~$i$, i.\,e., when the best fitness in the current population is~$i$.
To this end, we focus on the higher fitness levels ${i \ge n - n/\log n}$ where the probability of creating an offspring on the same level can be estimated nicely.
The time for reaching these higher fitness levels only constitutes a small-order term, compared to the claimed running time bound. The following lemma proves this claim in a more general setting than needed for the proof of Theorem~\ref{the:upper-bound-k-point-crossover-variable-p}. In particular, it holds for arbitrary tie-breaking rules and crossover operators.

\begin{lemma}
\label{lem:first-fitness-levels}
For every \mlGA implementing Algorithm~\ref{alg:Scheme-GA} with $\mu, \lambda = O(1), p_c = 1 - \Omega(1)$, and $p=c/n$ for a constant $c > 0$, using any initialization and any crossover operator, the expected time until a fitness level~$i \ge n - n/\log n$ is reached for the first time is $o(n \log n)$.
\end{lemma}
A proof is given in the appendix.

In the remainder of the section we focus on higher fitness levels~$i \ge n - n/\log n$ and specify the different cases on each such fitness level.
The cases~$i.1$, $i.2$, and $i.3$ are similar to the ones for uniform crossover, with additional conditions on the similarity of individuals in Cases~$i.2$ and $i.3$. We also have an additional error state that accounts for undesirable and unexpected behavior. We pessimistically assume that the error state cannot be left towards other cases on level~$i$.

\textbf{Case $i.1$:} The population contains an individual on level~$i$ and at least one individual on a lower fitness level.

\textbf{Case $i.2$:} The population contains $\mu$ copies of an individual~$x$ on level~$i$.

\textbf{Case $i.3$:} The population contains two search points $x, y$ with current best fitness~$i$, where $y$ resulted from a mutation of~$x$ and the Hamming distance of $x$ and~$y$ is~2.

\textbf{Case $i.$error:} An error state reached from any Case~$i.\cdot$ when the best fitness is~$i$ and none of the prior cases applies.

The difference to the analysis of uniform crossover is that in Case~$i.2$ we rely on the population collapsing to copies of a single individual. This helps to estimate the probability of creating a favourable parent-offspring pair in Case~$i.3$ as the \mlGA effectively only performs mutations of~$x$ while being in Case~$i.2$.

\begin{lemma}
\label{lem:time-i-1-2-3}
Consider any \mlGA as defined in Theorem~\ref{the:upper-bound-k-point-crossover-variable-p}, with parameters $2 \le \mu = O(1)$, $\lambda < \mu$, $p_c = o(1)$ and $p_c = \omega(1/\log n)$, $p = c/n$ for some constant~$c > 0$.
The total expected time spent in all Cases~$i.1$, $i.2$, and $i.3$ across all $i \ge n - n/\log n$ is at most
\[
\frac{n \ln n}{c \cdot e^{-c} \cdot (1+c)} + o(n \log n).
\]
\end{lemma}
\begin{proof}
We have already analyzed the expected time in Cases~$i.1$ and $i.2$, across all fitness levels. As in the proof of Theorem~\ref{the:upper-bound-uniform-crossover-variable-p-every-GA}, we use Lemma~\ref{lem:takeover} and get that the expected time spent in all Cases~$i.1$ is at most
\[
\frac{O((\mu + \lambda) n \log \mu)}{(1-p_c)(1-p)^n} = O(n).
\]

In Case~$i.2$ the algorithm behaves like the one using uniform crossover described in Theorem~\ref{the:upper-bound-uniform-crossover-variable-p-every-GA} as both crossover operators are working on identical individuals. As before, Case~$i.2$ is left if either a better offspring is created, or if a different offspring with $i$ ones is created. In the latter case, either Case~$i.3$ or the error state $i$.error is reached.
By the proof of Theorem~\ref{the:upper-bound-uniform-crossover-variable-p-every-GA} we know that the expected time spent in Cases~$i.2$ across all levels~$i$ is bounded by
\[
\lambda n + \frac{\ln(n^2 p + n) + 1+p}{p(1-p)^{n-1} \cdot (1+np)} = \frac{n \ln n}{c \cdot e^{-c} \cdot (1+c)} + O(n).
\]

Now we estimate the total time spent in all cases~$i.3$. As this time turns out to be comparably small, we allow ourselves to ignore that fact that not all these cases are actually reached.

Case~$i.3$ implies that the population contains a parent-offspring pair $x, y$ with Hamming distance~2. Consider the mutation that has created this offspring and note that this mutation flips each 1-bit and each 0-bit with the same probability. If $a, b$ with $a < b$ denote the bit positions where $x$ and $y$ differ, then $D := b - a$ is a random variable with support $\{1, \dots, n-1\}$. By the law of total expectation,
\begin{equation}
\label{eq:T-i-3-law-of-total-expectation}
\E{T_{i.3}} = \sum_{d=1}^{n-1} \E{T_{i.3} \mid D = d} \cdot \Prob(D = d).
\end{equation}
We first bound the conditional expectation by considering probabilities for improvements.
If $D = d$ then by crossover is successful if crossover is performed (probability~$p_c$), if the search point where bit~$a$ is 1 is selected as first parent (probability at least $1/\mu$), if the remaining search point in $\{x, y\}$ is selected as second parent (probability at least $1/\mu$), and if cutting points are chosen that lead to a fitness improvement. The latter event has probability at least $d(N-d)/(N(N-1))$ by Lemma~\ref{lem:crossover-successful}, with $N := n-1$. Finally, we need to assume that the following mutation does not destroy any fitness improvements (probability at least $(1-p)^n$). The probability of a successful crossover is then at least, using $d(N-d) = \min(d, N-d) \cdot \max(d, N-d) \ge \min(d, N-d) \cdot N/2$,
\[
\frac{p_c (1-p)^n}{\mu^2} \cdot \frac{d(N-d)}{N(N-1)} \ge
\frac{p_c (1-p)^n}{\mu^2} \cdot \frac{\min(d, (N-d))}{2(N-1)} \ge
\frac{p_c (1-p)^n}{\mu^2} \cdot \frac{\min(d, n-d) -1}{2n}.
\]
Another means of escaping from Case~$i.3$ is by not using crossover, but having mutation create an improvement. The probability for this is at least
\begin{equation}
\label{eq:improvement-by-mutation-main}
(1 - p_c) \cdot (n-i)p(1-p)^{n-1} \ge \gamma \cdot \frac{n-i}{n}
\end{equation}
for a constant~$\gamma > 0$.
Applying Lemma~\ref{lem:success-probability-lambda-offspring}, we
have
\begin{equation}
\label{eq:T-i-3-conditional-expectation}
\E{T_{i.3} \mid D = d} \le \lambda + \frac{1}{\frac{p_c (1-p)^n}{\mu^2} \cdot \frac{\min(d, n-d) -1}{2n} + \gamma \cdot \frac{n-i}{n}}.
\end{equation}
Note that this upper bound is non-increasing with $\min(d, n-d)$. We are therefore pessimistic when replacing $\min(d, n-d)$ by a different random variable that is stochastically dominated by it.
According to Lemma~\ref{lem:bit-distance-distribution}, ${\min\{d, n-d\}}$ dominates the uniform distribution on $\{1, \dots, n/4\}$ (assume w.\,l.\,o.\,g.\ that $n$ is a multiple of~4).
Combining this with~\eqref{eq:T-i-3-law-of-total-expectation} and~\eqref{eq:T-i-3-conditional-expectation} yields
\begin{align*}
\E{T_{i.3}} \le\;& \lambda + \sum_{d=1}^{n/4} \frac{1}{n/4} \cdot \frac{1}{\frac{p_c (1-p)^n}{\mu^2} \cdot \frac{d-1}{2n} + \gamma \cdot \frac{n-i}{n}}\\
\le\;& \lambda + O(\mu^2/p_c) \cdot \sum_{d=1}^{n/4} \frac{1}{d-1 + n-i}.
\end{align*}
The last sum is estimated as follows.
\begin{align*}
\sum_{d=1}^{n/4} \frac{1}{d-1 + n-i}
= \sum_{d=0}^{n/4-1} \frac{1}{d + n-i}
=\;& \frac{1}{n-i} + \sum_{d=1}^{n/4-1} \frac{1}{d + n-i}\\
\le\;& 1 + \int_{d=0}^{n/4} \frac{1}{d + n-i} \;\mathrm{d}d\\
=\;& 1 + \ln\left(1 + \frac{n/4}{n-i}\right).
\end{align*}
Along with $\lambda = O(1)$, $n/4 \le n$, and $O(\mu^2/p_c) = o(\log n)$, we get
\[
\E{T_{i.3}} \le o(\log n) \cdot \left(1 + \ln\left(1 + \frac{n}{n-i}\right)\right).
\]
For the sum $T_{\cdot, 3} = \sum_{i=0}^{n-1} T_{i, 3}$ we then have the following.
\begin{align*}
\E{T_{\cdot.3}} \le\;& o(n \log n) + o(\log n) \cdot
\sum_{i=0}^{n-1} \ln\left(1 + \frac{n}{n-i}\right)\\
=\;& o(n \log n) + o(\log n) \cdot \sum_{i=1}^{n} \ln\left(1 + \frac{n}{i}\right)\\
\le\;& o(n \log n) + o(\log n) \cdot
\int_{i=0}^{n} \ln\left(1 + \frac{n}{i}\right) \; \mathrm{d}i\\
=\;& o(n \log n)
\end{align*}
as the integral is $2 \ln(2) n$.
This completes the proof.
\end{proof}

The remainder of the proof is devoted to estimating the expected time spent in the error state. To this end we need to consider events that take the \mlGA ``off course'', that is, deviating from situations described in Cases~$i.1$, $i.2$, and~$i.3$.

Since Case~$i.3$ is based on offspring with Hamming distance~2 to their parents, one potential failure is that an offspring with fitness~$i$, but Hamming distance greater than~2 to its parent is being created. This probability is estimated in the following lemma.
\begin{lemma}
\label{lem:probability-mutation-on-same-level}
For $i \in \{1, \dots, n-1\}$ let $p^{(i)}$ denote the probability that standard bit mutation with mutation rate~$0 < p \le 1/2$ of a search point with $i$ 1\nobreakdash-bits creates a different offspring with $i$ 1\nobreakdash-bits.
If $i(n-i)p^2(1-p)^{-2} \le 1/2$ then
\[
i(n-i)p^2(1-p)^{n-2}
\;\le\; p^{(i)} \;\le\; i(n-i)p^2(1-p)^{n-2} \cdot \left(1+\frac{2i(n-i)p^2}{(1-p)^{2}}\right).
\]
The probability that, additionally, the offspring has Hamming distance larger than 2 to its parent is at most
\[
2i^2(n-i)^2p^4(1-p)^{n-4}.
\]
\end{lemma}
The proof is found in the appendix.

Another potential failure occurs if the population does not collapse to copies of a single search point, that is, the transition from Case~$i.1$ to Case~$i.2$ is not made. We first estimate the probability of mutation unexpectedly creating an individual with fitness~$i$.
\begin{lemma}
\label{lem:jump-to-level-i}
The probability that a standard bit mutation with mutation probability~${0 < p < 1}$ creates a search point with $i$ ones out of a parent with less than~$i$ ones, is at most
\[
p(n-i+1) \cdot e^{(pn)^2/4+1}.
\]
\end{lemma}
Note that for the special case $p=1/n$ Lemma~13 in~\cite{Doerr2012a} gives an upper bound of $(n-i+1)/n$ as the highest probability for a jump to fitness level~$i$ is attained if the parent is on level~$i-1$. However, for larger mutation probabilities this is no longer true in general; there are cases where the probability of jumping to level~$i$ is maximized for parents on lower fitness levels. Hence, a closer inspection of transition probabilities between different fitness levels is required, see the proof in the Appendix.

Using Lemma~\ref{lem:jump-to-level-i}, we can now estimate the probability of the \mlGA not collapsing to copies of a single search point as described in Case~$i.2$.
\begin{lemma}
\label{lem:takeover-to-copies}
Consider any \mlGA as defined in Theorem~\ref{the:upper-bound-k-point-crossover-variable-p}, with parameters $2 \le \mu = O(1)$, $\lambda < \mu$, $p_c = o(1)$ and $p_c = \omega(1/\log n)$, $p = c/n$ for some constant~$c > 0$
and fix a fitness level~$i < n$. The probability that the \mlGA will reach a population containing different individuals with fitness~$i$ before either reaching a population containing only copies of the same individual on level~$i$ or reaching a higher fitness level, is at most
\[
O(\mu \log \mu) \cdot \left(p_c + \frac{n-i}{n}\right).
\]
\end{lemma}
\begin{proof}
We show that there is a good probability of repeatedly creating clones of individuals with fitness~$i$ (or finding an improvement) and avoiding the following \emph{bad} event. A bad event happens if an individual on fitness level~$i$ is created in one offspring creation by means other than through cloning an existing individual on level~$i$.

The probability of a bad event is bounded as follows.
In case crossover is being used, which happens with probability $p_c$,
we bound the probability of a bad event by the trivial bound~1. Otherwise, such an individual needs to be created through mutation from either a worst fitness level, or by mutating a parent on level~$i$. The probability for the former is bounded from above by Lemma~\ref{lem:jump-to-level-i}. The probability for the latter is at most $p(n-i)$ as it is necessary to flip one out of $n-i$ 0\nobreakdash-bits.
Using $n-i+1 \le 2(n-i)$, the probability of a bad event on level~$i$ is hence bounded from above by
\begin{align*}
& p_c + (1-p_c) \cdot \left(p(n-i+1) \cdot e^{(pn)^2/4+1} + p(n-i)\right)\\
\le\;& p_c + \left(\frac{2c(n-i)}{n} \cdot e^{c^2/4+1} + \frac{c(n-i)}{n}\right) = p_c + \kappa \cdot \frac{n-i}{n},
\end{align*}
where $\kappa := 2c \cdot e^{c^2/4+1} + c$ is a constant.
The \mlGA will only reach a population containing different individuals with fitness~$i$ as stated if a bad event happens before the population has collapsed to copies of a single search point or moved on to a higher fitness level.

Consider the first generation where an individual of fitness~$i$ is reached for the first time. Since it might be possible to create several such individuals in one generation, we consider all offspring creations being executed sequentially and consider the possibility of bad events for all offspring creations following the first offspring on level~$i$.
Let $X$ be the number of function evaluations following this generation, before all individuals in the population have fitness at least~$i$. By Lemma~\ref{lem:takeover} we have
\[
\E{X} = \mathord{O}\mathord{\left(\frac{(\mu + \lambda) \log \mu}{(1-p_c)(1-p)^n}\right)} = O(\mu \log \mu).
\]
Considering up to $\lambda$ further offspring creations in the first generation leading to level~$i$, and completing the generation at the end of the $X$ function evaluations, we have less than $X + 2\lambda$ trials for bad events. The probability that one of these is bad is bounded by
\begin{align*}
\left(\sum_{t=1}^{X} t \cdot \Prob(X = t) + 2\lambda\right) \cdot \left(p_c + \kappa \cdot \frac{n-i}{n}\right)
=\;& (\E{X} + 2\lambda) \cdot \left(p_c + \kappa \cdot \frac{n-i}{n}\right)\\
=\;& O(\mu \log \mu) \cdot \left(p_c + \kappa \cdot \frac{n-i}{n}\right).
\end{align*}
Absorbing $\kappa$ in the $O$-term yields the claimed result.
\end{proof}

Now we are prepared to estimate the expected time spent in all error states $i.$error for $i \ge n - n/\log n$.
\begin{lemma}
\label{lem:time-i-error}
Consider any \mlGA as defined in Theorem~\ref{the:upper-bound-k-point-crossover-variable-p}, with parameters $2 \le \mu = O(1)$, $\lambda < \mu$, $p_c = o(1)$ and $p_c = \omega(1/\log n)$, $p = c/n$ for some constant~$c > 0$.
The expected time spent in all states $i.$error, for $i \ge n - n/\log n$, is at most
\[
O(n + p_c \cdot n \ln n) = o(n \log n).
\]
\end{lemma}
\begin{proof}
The \mlGA only spends time in an error state if it is actually reached. So we first calculate the probability that state $i.$error is reached from either Case~$i.1$, $i.2$, or $i.3$.

Lemma~\ref{lem:takeover-to-copies} states that the probability of reaching a population with different individuals on level~$i$ before reaching Case~$i.2$ or a better fitness level is
\[
O(\mu \log \mu) \cdot \left(p_c + \frac{n-i}{n}\right) = \mathord{O}\mathord{\left(p_c + \frac{n-i}{n}\right)}.
\]
We pessimistically ignore the possibility that Case~$i.3$ might be reached if this happens; thus the above is an upper bound for the probability of reaching $i.$error from Case~$i.1$.

Recall that in Case~$i.2$ all individuals are identical, so crossover has no effect and the \mlGA only performs mutations.
First consider the case $\lambda=1$. Note that $i \ge n - n/\log n$, along with $p=c/n$, implies that $i(n-i)p^2(1-p)^{-2} \le 1/2$, hence Lemma~\ref{lem:probability-mutation-on-same-level} is in force.
According to Lemma~\ref{lem:probability-mutation-on-same-level} the probability of leaving Case~$i.2$ by creating a different individual with fitness~$i$ is at least $i(n-i)p^2(1-p)^{n-2}$. The probability of doing this with an offspring of Hamming distance greater than 2 to its parent is at most $2i^2(n-i)^2 p^4 (1-p)^{n-4}$ (second statement of Lemma~\ref{lem:probability-mutation-on-same-level}). So the conditional probability of reaching the error state when leaving Case~$i.2$ towards another case on level~$i$ is at most
\begin{equation}
\label{eq:error-from-i.2}
\frac{2i^2(n-i)^2 p^4 (1-p)^{n-4}}{i(n-i)p^2(1-p)^{n-2}} = 2i(n-i)p^2(1-p)^{-2}.
\end{equation}
In case $\lambda > 1$ note that Case~$i.3$ is reached in case there is a single offspring with fitness~$i$ and Hamming distance~$2$ to its parent. Such an offspring is guaranteed to survive as we assume $\lambda < \mu$ and offspring with many duplicates are removed first. Thus, in case several offspring with fitness~$i$ and differing from their parent are created, \emph{all} of them need to have Hamming distance larger than~2 in order to reach $i.$error from Case~$i.2$. This probability decreases with increasing $\lambda$, hence the probability bound~\eqref{eq:error-from-i.2} also holds for $\lambda > 1$.

Finally, Case~$i.3$ implies that there exists a parent-offspring pair $x, y$ with Hamming distance~2. In a new generation these two offspring -- or at least one copy of each -- will always survive: individuals with multiple duplicates are removed first, and in case among current parents and offspring more than $\mu$ individuals exist with no duplicates, $x$ and $y$ will be preferred over newly created offspring. So the probability of reaching the error state from Case~$i.3$ is 0.

In case the error state is reached, according to~\eqref{eq:improvement-by-mutation-main} we have a probability of at least $\gamma \cdot \frac{n-i}{n}$ of finding a better individual in one offspring creation, for a constant~$\gamma > 0$. Using Lemma~\ref{lem:success-probability-lambda-offspring} as before, this translates to at most $\lambda + \frac{1}{\gamma} \cdot \frac{n}{n-i}$ expected function evaluations. So the expected time spent in Case~$i.$error is at most
\begin{align*}
& \lambda + \left(2i(n-i)p^2(1-p)^{-2} + \mathord{O}\mathord{\left(p_c + \frac{n-i}{n}\right)}\right) \cdot \frac{1}{\gamma} \cdot \frac{n}{n-i}\\
=\;& \lambda + \left(\frac{2}{\gamma} \cdot inp^2(1-p)^{-2} + \mathord{O}\mathord{\left(p_c \cdot \frac{n}{n-i} + 1\right)}\right)\\
 =\;& \mathord{O}\mathord{\left(p_c \cdot \frac{n}{n-i}\right)} + O(1)
\end{align*}
as both $\lambda = O(1)$ and $inp^2(1-p)^{-2} \le (pn)^2 \cdot (1-p)^{-2} = O(1)$.
The total expected time across all error states is at most
\[
\mathord{O}\mathord{\left(n + p_c \cdot n \cdot \sum_{i=0}^{n-1} \frac{1}{n-i}\right)} = O(n + p_c \cdot n \ln n). \qedhere
\]
\end{proof}

Now Theorem~\ref{the:upper-bound-k-point-crossover-variable-p} follows from all previous lemmas.
\begin{proof}[Proof of Theorem~\ref{the:upper-bound-k-point-crossover-variable-p}]
The claimed upper bound now follows from adding the upper bounds on the expected time on the smaller fitness levels (Lemma~\ref{lem:first-fitness-levels}) to the expected times spent in all considered cases (Lemma~\ref{lem:time-i-1-2-3} and Lemma~\ref{lem:time-i-error}).
\end{proof}

We believe that some of the technical conditions from Theorem~\ref{the:upper-bound-k-point-crossover-variable-p} involving $\mu, \lambda$, and $p_c$ could be relaxed if it was possible to generalize Lemmas~\ref{lem:crossover-successful} and~\ref{lem:bit-distance-distribution} towards more than 2 differing bits between individuals of equal fitness.

Figure~\ref{fig:different-functions}, discussed in the following Section~\ref{sec:extensions}, presents further experiments and statistical tests, which includes a comparison of uniform crossover and $k$\nobreakdash-point crossover in the greedy \TwoGA.

\section{Extensions to Other Building-Block Functions}
\label{sec:extensions}

\subsection{Royal Roads and Monotone Polynomials}
\label{sec:experiments}

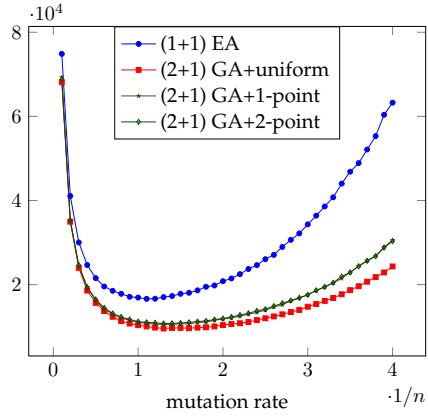
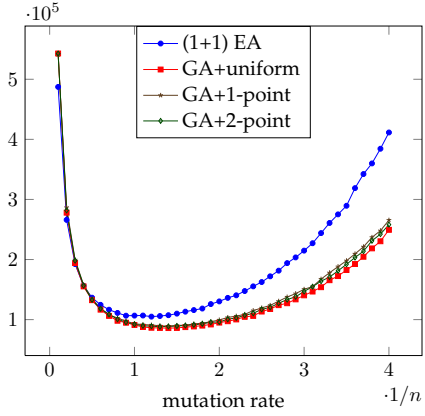
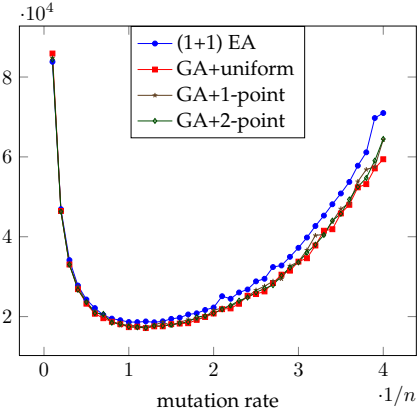
\begin{figure}[tb]
\centering
\subfigure[\ONEMAX]{
\begin{tikzpicture}[scale=0.77]
\begin{axis}[legend cell align=left, xlabel={mutation rate},
xtick scale label code/.code={$\cdot 1/n$},
]
\pgfplotstableread[header=false]{experiments-onemax/(1+1)EA-onemax-1000runs-n=1000-p=0.1-by-n,...,4-by-n.txt}\pOneOne
\pgfplotstableread[header=false]{experiments-onemax/greedy-GA-fastxover-onemax-1000runs-n=1000-p=0.1-by-n,...,4-by-n.txt}\pGreedyGA
\pgfplotstableread[header=false]{experiments-onemax/greedy-GA-1-point-crossover-onemax-1000runs-n=1000-p=0.1-by-n,...,4-by-n.txt}\pGAonept
\pgfplotstableread[header=false]{experiments-onemax/greedy-GA-2-point-crossover-onemax-1000runs-n=1000-p=0.1-by-n,...,4-by-n.txt}\pGAtwopt
\addplot table[x expr=0.1*0.001*(\thisrow{1}), y index=7]{\pOneOne};
\addplot table[x expr=0.1*0.001*(\thisrow{1}), y index=7]{\pGreedyGA};
\addplot table[x expr=0.1*0.001*(\thisrow{1}), y index=7]{\pGAonept};
\addplot table[x expr=0.1*0.001*(\thisrow{1}), y index=7]{\pGAtwopt};
\legend{\EA, \TwoGA{}+uniform, \TwoGA{}+1-point, \TwoGA{}+2-point}
\end{axis}
\end{tikzpicture}
}
\\
\subfigure[Royal road]{
\begin{tikzpicture}[scale=0.77]
\begin{axis}[legend cell align=left, xlabel={mutation rate},
xtick scale label code/.code={$\cdot 1/n$},
]
\pgfplotstableread[header=false]{experiments-onemax/(1+1)EA-RoyalRoad5fast-1000runs-n=1000-p=0.1-by-n,...,4-by-n.txt}\pOneOne
\pgfplotstableread[header=false]{experiments-onemax/greedy-GA-RoyalRoad5fast-1000runs-n=1000-p=0.1-by-n,...,4-by-n.txt}\pGreedyGA
\pgfplotstableread[header=false]{experiments-onemax/greedy-GA-1-point-crossover-RoyalRoad5fast-1000runs-n=1000-p=0.1-by-n,...,4-by-n.txt}\pGAonept
\pgfplotstableread[header=false]{experiments-onemax/greedy-GA-2-point-crossover-RoyalRoad5fast-1000runs-n=1000-p=0.1-by-n,...,4-by-n.txt}\pGAtwopt
\addplot table[x expr=0.1*0.001*(\thisrow{1}), y index=7]{\pOneOne};
\addplot table[x expr=0.1*0.001*(\thisrow{1}), y index=7]{\pGreedyGA};
\addplot table[x expr=0.1*0.001*(\thisrow{1}), y index=7]{\pGAonept};
\addplot table[x expr=0.1*0.001*(\thisrow{1}), y index=7]{\pGAtwopt};
\legend{\EA, GA+uniform, GA+1-point, GA+2-point}
\end{axis}
\end{tikzpicture}
}
\subfigure[Random polynomials]{
\begin{tikzpicture}[scale=0.77]
\begin{axis}[legend cell align=left, xlabel={mutation rate},
xtick scale label code/.code={$\cdot 1/n$},
]
\pgfplotstableread[header=false]{experiments-onemax/(1+1)EA-RandomPolynomials1000-5-1000runs-p=0.1-by-n,...,4-by-n.txt}\pOneOne
\pgfplotstableread[header=false]{experiments-onemax/greedyGA-RandomPolynomials1000-5-1000runs-p=0.1-by-n,...,4-by-n.txt}\pGreedyGA
\pgfplotstableread[header=false]{experiments-onemax/greedyGA-1ptxover-RandomPolynomials1000-5-1000runs-p=0.1-by-n,...,4-by-n.txt}\pGAonept
\pgfplotstableread[header=false]{experiments-onemax/greedyGA-2ptxover-RandomPolynomials1000-5-1000runs-p=0.1-by-n,...,4-by-n.txt}\pGAtwopt
\addplot table[x expr=0.1*0.001*(\thisrow{1}), y index=7]{\pOneOne};
\addplot table[x expr=0.1*0.001*(\thisrow{1}), y index=7]{\pGreedyGA};
\addplot table[x expr=0.1*0.001*(\thisrow{1}), y index=7]{\pGAonept};
\addplot table[x expr=0.1*0.001*(\thisrow{1}), y index=7]{\pGAtwopt};
\legend{\EA, GA+uniform, GA+1-point, GA+2-point}
\end{axis}
\end{tikzpicture}
}
\caption{Average optimization times over 1000 runs for \EA and the greedy~\TwoGA with various crossover operators on functions with $n=1000$ bits: \ONEMAX, a royal road function with block size~5, and random polynomials with 1000 unweighted monomials of degree~5. The mutation rate is~$c/n$ with $c \in \{0.1, 0.2, \dots, 4\}$.}
\label{fig:different-functions}
\end{figure}

\begin{table*}
\centerline{
\footnotesize
\begin{tabular}{lllllccc}
& & \EA & uniform & 1-point \\
\hline
\ONEMAX & uniform & $p < 10^{-3}$ & & \\
& 1-point & $p < 10^{-3}$  & $p < 10^{-3}$ for $c \ge 0.4$ & \\
& 2-point & $p < 10^{-3}$  & $p < 10^{-3}$ for $c \ge 0.3$ & $p > 10^{-3}$ (11 ex.)\\
\hline
Royal road & uniform & $p<10^{-3}$ for $c \ge 0.6$ & & \\
& 1-point & $p<10^{-3}$ for $c \ge 0.6$ & $p < 10^{-3}$ for $c \ge 0.8$ (1 ex.) & \\
& 2-point & $p<10^{-3}$ for $c \ge 0.6$ & $p < 10^{-3}$ for $c \ge 1.4$ (5 ex.) & $p > 10^{-3}$ (6 ex.)\\
\hline
Random & uniform & $p<10^{-3}$ (1 ex.) & & \\
Polynomial & 1-point & $p<10^{-3}$ for $c \ge 0.3$ (3 ex.) & $p > 10^{-3}$ (13 ex.) & \\
& 2-point & $p < 10^{-3}$ for $c \ge 0.4$ (1 ex.) & $p > 10^{-3}$ (6 ex.) & $p > 10^{-3}$ (6 ex.)\\
\hline
\end{tabular}
}
\caption{Summary of the results of two-sided Mann-Whitney $U$ tests on the data from Figure~\ref{fig:different-functions}. For each function the table shows pairwise comparisons between the \EA and the greedy~GA with uniform, 1-point, and 2-point crossover, resp. Here $p$ is the $p$-value output by the statistics package R (version 2.8.1) and $c$ is the constant in the mutation rate $c/n$. Each cell describes a rule for~$p$ subject to a minimum value of~$c$ and gives the number of exceptions from this rule where applicable.}
\label{tab:statistical-tests}
\end{table*}

So far, our theorems and proofs have been focused on \ONEMAX only. This is because we do have very strong results about the performance of EAs on \ONEMAX at hand. However, the insights gained stretch far beyond \ONEMAX. Royal road functions generally consist of larger blocks of bits. All bits in a block need to be set to~1 in order to contribute to the fitness; otherwise the contribution is~0. All blocks contribute the same amount to the fitness, and the fitness is just the sum of all contributions.

The fundamental insight we have gained for neutral mutations also applies to royal road functions. If there is a mutation that completes one block, but destroys another block, this is a neutral mutation and the offspring will be stored in the population of a \mlGA. Then crossover can recombine all finished blocks in the same way as for \ONEMAX. The only difference is that the destroyed block may evolve further. More neutral mutations can occur that only alter bits in the destroyed block. Then the population can be dominated by many similar solutions, and it becomes harder for crossover to find a good pair for recombination. However, as crossover has generally a very high probability of finding improvements, the last effect probably plays only a minor role.

A theoretical analysis of general royal roads up to the same level of detail as for \ONEMAX is harder, but not impossible. So far results on royal roads and monotone polynomials have been mostly asymptotic~\cite{Wegener2005c,Doerr2013}. Only recently, Doerr and K{\"u}nnemann~\cite{Doerr2013b} presented a tighter runtime analysis of offspring populations for royal road functions, which may lend itself to a generalization of our results on \ONEMAX in future work.

For now, we use experiments to see whether the performance is similar to that on \ONEMAX. We use royal roads with $n=1000$ bits and block size~5, i.\,e., we have 200 pairwise disjoint blocks of 5 bits each. We also consider random monotone polynomials. Instead of using disjoint blocks, we use $1000$ monomials of degree~5 (conjunctions of 5 bits): each monomial is made up of 5 bit positions chosen uniformly at random, without replacement. This leads to a function similar to royal roads, but ``blocks'' are broken up and can share bits; bit positions are completely random. Figure~\ref{fig:different-functions} shows the average optimization times in 1000 runs on all these functions, for the \EA and the greedy~\TwoGA with uniform, 1-point, and 2-point crossover. We chose the last two because $k$\nobreakdash-point crossovers for odd~$k$ treat ends of bit strings differently from those for even~$k$: for odd $k$ two bits close to opposite ends of a bitstring have a high probability to be taken from different parents, whereas for even $k$ there is a high chance that both will be taken from the same parent (cf.\ Lemma~\ref{lem:crossover-successful} for $k=2$ and the special case of~$k=1$).

For consistency and simplicity, we use $p_c=1$ and the tie-breaking rule dup-rnd in all settings, that is, ties in fitness are broken towards minimum numbers of duplicates and any remaining ties are broken uniformly at random. For \ONEMAX this does not perfectly match the conditions of Theorem~\ref{the:upper-bound-k-point-crossover-variable-p} as they require a lower crossover probability, $p_c = o(1)$, and tie-breaking rule dup-old. But the experiments show that $k$\nobreakdash-point crossover is still effective when these conditions are not met.

On \ONEMAX both $k$\nobreakdash-point crossovers are better than the \EA, but slightly worse than uniform crossover. This is in accordance with the observation from our analyses that improvements with $k$\nobreakdash-point crossover might be harder to find, in case the differing bits are in close proximity.

For royal roads the curves are very similar. The difference between the \EA and the greedy~\TwoGA is just a bit smaller.
For random polynomials there are visible differences, albeit smaller.
Mann-Whitney $U$ tests confirm that wherever there is a noticeable gap between the curves, there is a statistically significant difference on a significance level of $0.001$.
The outcome of Mann-Whitney $U$ tests is summarized in Table~\ref{tab:statistical-tests}.

For very small mutation rates $c/n$ the tests were not significant. For mutation rates no less than $0.6/n$ all differences between the \EA and all greedy~\TwoGA{}s were statistically significant, apart from a few exceptions on random polynomials. For \ONEMAX the difference between uniform crossover vs.\ $k$\nobreakdash-point crossover was significant for $c \ge 0.4$. For royal roads the majority of such comparisons showed statistical significance, with a number of exceptions. However, for random polynomials the majority of comparisons were not statistically significant. Most comparisons between 1-point and 2-point crossover did not show statistical significance.

These findings give strong evidence that the insights drawn from the analysis on \ONEMAX transfer to broader classes of functions where building blocks need to be assembled.

\subsection{Linear Functions}
\label{sec:linear-functions}

Another interesting question is in how far the theoretical analyses in this work extend to cases where building blocks have different weights. The simplest such case is the class of linear functions, defined as
\[
f(x) = \sum_{i=1}^n w_i x_i
\]
where $w_i > 0$ are positive real-valued weights.

Doerr, Doerr, and Ebel~\cite{Doerr2013a} provided empirical evidence that their (1+($\lambda$, $\lambda$))~EA is faster than the \EA on linear functions with weights drawn uniformly at random from $[1, 2]$.

It is an open question whether this also holds for more common GAs, that is, those implementing Algorithm~\ref{alg:Scheme-GA}. Experiments in~\cite{Doerr2013a} on the greedy \TwoGA found that on random linear functions ``no advantage of the \TwoGA over the \EA is visible''. We provide an explanation for this observation and reveal why the \TwoGA is not well suited for weighted building blocks, whereas other GAs might be.

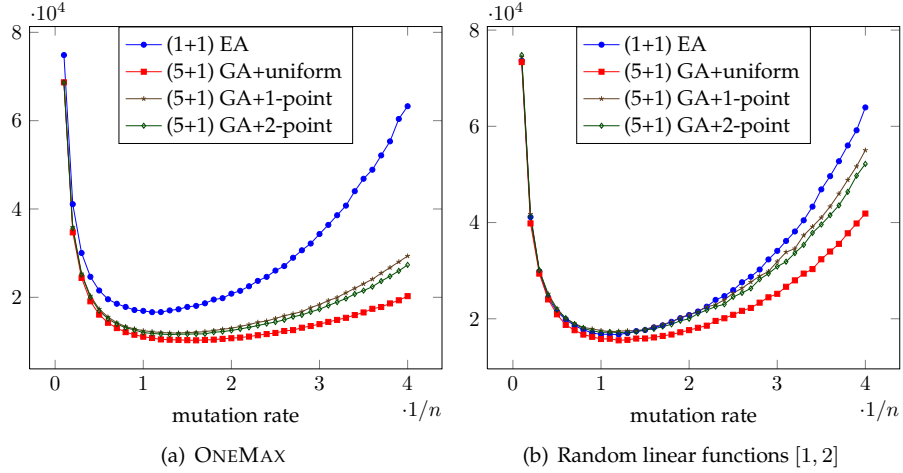
\begin{figure*}[hbt]
\centering
\subfigure[\ONEMAX]{
\begin{tikzpicture}[scale=0.8]
\begin{axis}[legend cell align=left, xlabel={mutation rate},
xtick scale label code/.code={$\cdot 1/n$},
]
\pgfplotstableread[header=false]{experiments-onemax/(1+1)EA-onemax-1000runs-n=1000-p=0.1-by-n,...,4-by-n.txt}\pOneOne
\pgfplotstableread[header=false]{experiments-onemax/greedy-GA-fastxover-onemax-1000runs-n=1000-p=0.1-by-n,...,4-by-n.txt}\pGreedyGA
\pgfplotstableread[header=false]{experiments-onemax/(5+1)GA-duptiebreak-onemax-1000runs-n=1000-p=0.1-by-n,...,4-by-n.txt}\pFiveGA
\pgfplotstableread[header=false]{experiments-onemax/(5+1)GA-1ptxover-duptiebreak-onemax-1000runs-n=1000-p=0.1-by-n,...,4-by-n.txt}\pFiveGAOnept
\pgfplotstableread[header=false]{experiments-onemax/(5+1)GA-2ptxover-duptiebreak-onemax-1000runs-n=1000-p=0.1-by-n,...,4-by-n.txt}\pFiveGATwopt
\addplot table[x expr=0.1*0.001*(\thisrow{1}), y index=7]{\pOneOne};
\addplot table[x expr=0.1*0.001*(\thisrow{1}), y index=7]{\pFiveGA};
\addplot table[x expr=0.1*0.001*(\thisrow{1}), y index=7]{\pFiveGAOnept};
\addplot table[x expr=0.1*0.001*(\thisrow{1}), y index=7]{\pFiveGATwopt};
\legend{\EA, {(5+1)~GA+uniform}, {(5+1)~GA+1-point}, {(5+1)~GA+2-point}}
\end{axis}
\end{tikzpicture}
}
\subfigure[Random linear functions {$[1, 2]$}]{
\begin{tikzpicture}[scale=0.8]
\begin{axis}[legend cell align=left, xlabel={mutation rate},
xtick scale label code/.code={$\cdot 1/n$},
]
\pgfplotstableread[header=false]{experiments-onemax/(1+1)EA-Randomlinear1-2-1000runs-n=1000-p=0.1-by-n,...,4-by-n.txt}\pOneOne
\pgfplotstableread[header=false]{experiments-onemax/(5+1)GA-duptiebreak-Randomlinear1-2-1000runs-n=1000-p=0.1-by-n,...,4-by-n.txt}\pFiveGA
\pgfplotstableread[header=false]{experiments-onemax/(5+1)GA-1ptxover-duptiebreak-Randomlinear1-2-1000runs-n=1000-p=0.1-by-n,...,4-by-n.txt}\pFiveGAOnept
\pgfplotstableread[header=false]{experiments-onemax/(5+1)GA-2ptxover-duptiebreak-Randomlinear1-2-1000runs-n=1000-p=0.1-by-n,...,4-by-n.txt}\pFiveGATwopt
\addplot table[x expr=0.1*0.001*(\thisrow{1}), y index=7]{\pOneOne};
\addplot table[x expr=0.1*0.001*(\thisrow{1}), y index=7]{\pFiveGA};
\addplot table[x expr=0.1*0.001*(\thisrow{1}), y index=7]{\pFiveGAOnept};
\addplot table[x expr=0.1*0.001*(\thisrow{1}), y index=7]{\pFiveGATwopt};
\legend{\EA, {(5+1)~GA+uniform}, {(5+1)~GA+1-point}, {(5+1)~GA+2-point}}
\end{axis}
\end{tikzpicture}
}
\caption{Average optimization times over 1000 runs for \EA and a (5+1)~GA with uniform parent selection and various crossover operators on functions with $n=1000$ bits: \ONEMAX and random linear functions with weights drawn independently, uniformly at random from $[1, 2]$, and anew for each run.}
\label{fig:runtime-linear-functions}
\end{figure*}

The reason why the \TwoGA behaves like the \EA in the presence of weights is that in case the current population of the \TwoGA contains two members with different fitness, the \TwoGA ignores the inferior one. So it behaves as if the population only contained the fitter individual. Since the \TwoGA will select the fitter individual twice for crossover, followed by mutation, it essentially just mutates the fitter individual. This behavior of the \TwoGA then equals that of a \EA working on the fitter individual.

The \TwoGA is more efficient than the \EA on \ONEMAX (and other building-block functions where all building blocks are equally important) as it can easily generate and store individuals with equal fitness in the population, and recombine their different building blocks. However, in the presence of weights, chances of creating individuals of equal fitness might be very slim, and then the \TwoGA behaves like the \EA.

\begin{theorem}
\label{the:TwoGA-linear}
As long as the population of the \TwoGA does not contain two different individuals with the same fitness, the \TwoGA is equivalent to the \EA.

On functions where all search points have different fitness values, the \TwoGA is equivalent to the \EA. This includes linear functions with extreme weights like
\[
\mathrm{BinVal}(x) := \sum_{i=1}^n 2^{n-i} x_i
\]
and, more generally, functions where $w^{(i)} > \sum_{j=i+1}^n w^{(j)}$ for all $1 \le i \le n$, where $w^{(i)}$ denotes the $i$-th largest weight. It also includes, almost surely, random linear functions with weights being drawn from some real-valued interval $[a, b]$ with $a < b$.
\end{theorem}
\begin{proof}
The first two statements have been established in the preceding discussion.

For functions where $w^{(i)} > \sum_{j=i+1}^n w^{(j)}$ for all $1 \le i \le n$, all search points with a 1 on the bit of weight $w^{(i)}$ have a higher fitness than all search points where this bit is 0, provided that all bits with larger weights are being fixed. It follows inductively that all search points have different fitness values.

For random linear functions consider the function being constructed sequentially by adding new bits with randomly drawn weights. Assume that after adding $i$ bits, all $2^i$ bit patterns have different fitness values. This is trivially true for 0 bits. When adding a new bit~$i+1$, a fitness value can only be duplicated with these $i+1$ bits if the $i+1$-st weight is equal to any selection of weights from the first $i$ bits. Since there are at most $2^{i}$ selections, which is finite, the $i+1$-st weight will almost surely be different from all of these. The statement then follows by induction.
\end{proof}

In a sense, the \TwoGA is not able to benefit from crossover in the settings from Theorem~\ref{the:TwoGA-linear} since its greedy parent selection suppresses diversity in the population.

So, in order for a GA to benefit from crossover, the population needs to be able to maintain and select individuals with different building blocks and slightly different fitness values for long enough, so that crossover has a good chance of combining those building blocks.
The (1+($\lambda$, $\lambda$))~EA~\cite{Doerr2013a} achieves this using a cleverly designed two-stage offspring creation process: mutation first creates diversity and the best among $\lambda$ mutants is retained and recombined with its parent $\lambda$ times. However, this does not explain why crossover is beneficial in common GA designs.

A promising common GA design does not need to be sophisticated -- Figure~\ref{fig:runtime-linear-functions} shows that already a simple (5+1)~GA with uniform parent selection performs significantly better than the \EA (and hence the greedy \TwoGA). The benefit of crossover is smaller than that on \ONEMAX, but the main qualitative observations are the same: the average optimization time is smaller with crossover, and mutation rates slightly larger than~$1/n$ further improve performance.

One-sided Mann-Whitney $U$ tests on a significance level of $0.001$ showed that the (5+1)~GA with uniform crossover was significantly faster than the \EA on random linear functions, for mutation rates no less than $0.6/n$.
Both $k$\nobreakdash-point crossovers gave mixed results: they were slower than the \EA for low mutation rates ($0.4 \le c \le 1.2$, except for $c=0.6$, for 1-point crossover and $0.9 \le c \le 1.1$ for 2-point crossover), but faster for high mutation rates ($c=2.3$ and $c \ge 2.6$ for 1-point crossover, $c=2.0$ and $c \ge 2.2$ for 2-point crossover).

This shows that uniform crossover can speed up building-block assembly for weighted building blocks -- albeit not for all \mlGA{}s, and in particular not for the greedy \TwoGA. Proving this rigorously for random or arbitrary linear functions remains a challenging open problem, and so is identifying characteristics of \mlGA{}s for which crossover is beneficial in these cases.

\section{Conclusions and Future Work}
\label{sec:conclusions}

We have demonstrated rigorously and intuitively that crossover can speed up building block assembly on \ONEMAX, with evidence that the same holds for a broad class of functions. The basic insight is that mutations that create new building blocks while destroying others can still be useful: mutants can be stored in the population and lead to a successful recombination with their parents in a later generation. This effect makes every \mlGA{} with cut selection and moderate population sizes twice as fast as every mutation-based EA on \ONEMAX. In other words, adding crossover to any such \mlEA halves the expected optimization time (up to small-order terms).
This applies to uniform crossover and to $k$\nobreakdash-point crossover, for arbitrary values of~$k$.

Furthermore, we have demonstrated how to analyze parent and offspring populations as in \mlEA{}s and \mlGA{}s. As long as both $\mu$ and $\lambda$ are moderate, so that exploitation is not slowed down, we obtained essentially the same results for arbitrary \mlGA{}s as for the simple greedy \TwoGA analyzed in~\cite{Sudholt2012b}. This work provides novel techniques for the analysis of ($\mu$+$\lambda$)-type algorithms, including Lemmas~\ref{lem:takeover} and~\ref{lem:success-probability-lambda-offspring}, which may prove useful in further studies of EAs.

Another intriguing conclusion following naturally from our analysis is that the optimal mutation rate for GAs such as the greedy \TwoGA changes from $1/n$ to $(1+\sqrt{5})/2 \cdot 1/n \approx 1.618/n$ when using uniform crossover. This is simply because neutral mutations and hence multi-bit mutations become more useful. Experiments are in perfect accordance with the theoretical results for \ONEMAX. For other functions like royal roads and random polynomials they indicate that the performance differences also hold in a much more general sense. We have empirical evidence that this might also extend to linear functions, and weighted building blocks in general, albeit this does not apply to the greedy \TwoGA. The discussion from Section~\ref{sec:linear-functions} has shown that the population must be able to store individuals with different building blocks for long enough so that crossover can combine them, even though some individuals might have inferior fitness values and be subject to replacement.

Our results give novel, intuitive and rigorous answers to a question that has been discussed controversially for decades.

There are plenty of avenues for future work. We would like to extend the theoretical analysis of \mlGA{}s to royal road functions and monotone polynomials. Also investigating weighted building blocks, like in linear functions, is an interesting and challenging topic for future work.

Our \mlGA{}s benefit from crossover and an increased mutation rate because cut selection removes offspring with inferior fitness. As such, cut selection counteracts disruptive effects of crossover and an increase of the mutation rate. The situation is entirely different in generational GAs, where Ochoa, Harvey, and Buxton reported that introducing crossover can \emph{decrease} the optimal mutation rate~\cite{Ochoa1999}. Future work could deal with complementing these different settings and investigating the balance between selection pressure for replacement selection and the optimal mutation rate.


\subsubsection*{Acknowledgments}
The author was partially supported by EPSRC grant EP/D052785/1 while being a member of CERCIA, University of Birmingham, UK. The research leading to these results has received funding from the
European Union Seventh Framework Programme (FP7/2007-2013) under grant
agreement no 618091 (SAGE). The author would like to thank the reviewers for their detailed and constructive comments that helped to improve the manuscript.

\bibliographystyle{abbrv}
\bibliography{onemax-short}

\appendix

\section{Appendix}

The appendix contains proofs of lemmas omitted from the main part.

\begin{proof}[Proof of Lemma~\ref{lem:first-fitness-levels}]
If the current population has a best individual of fitness~$j < i$, by Lemma~\ref{lem:takeover} after an expected number of $O((\mu + \lambda) \log \mu) = O(1)$ function evaluations all individuals will have fitness at least~$j$. Then one offspring creation results in an improvement if no crossover is being used, and mutation flips exactly one out of $n-j$ 0\nobreakdash-bits. The probability for this event is
\begin{equation*}
\label{eq:improvement-by-mutation-appendix}
(1 - p_c) \cdot (n-j)p(1-p)^{n-1} \ge \gamma \cdot \frac{n-j}{n}
\end{equation*}
for some constant $\gamma > 0$, due to our conditions for $p$ and~$p_c$.

Using Lemma~\ref{lem:success-probability-lambda-offspring}, the expected time until a fitness level $i \ge n - n/\log n$ is reached for the first time is therefore at most
\begin{align*}
\sum_{j=0}^{n - (n/\log n) - 1} \left(O(1) + \lambda + \frac{n}{\gamma(n-j)}\right)
=\;& O(n) + \frac{n}{\gamma} \cdot \sum_{j=(n/\log n) + 1}^{n} \frac{1}{j}\\
\le\;& O(n) + \frac{n}{\gamma} \cdot \int_{j=n/\log n}^{n} \frac{1}{j} \;\mathrm{d}j\\
=\;& O(n) + \frac{n}{\gamma} \cdot \left(\ln n - \ln(n/\log n)\right)\\
=\;& O(n) + \frac{n}{\gamma} \cdot \ln(\log n)\\
=\;& o(n \log n).\qedhere
\end{align*}
\end{proof}

\begin{proof}[Proof of Lemma~\ref{lem:probability-mutation-on-same-level}]
In order to create a different search point on the same fitness level, there must be some integer $\ell \in \{1, \dots, \min\{i, n-i\}\}$ such that $\ell$ 1\nobreakdash-bits flip to~0 and $\ell$ 0\nobreakdash-bits flip to~1. This is a necessary and sufficient condition, so
\begin{equation}
\label{eq:probability-bound-p(i)}
p^{(i)} = \sum_{\ell=1}^{\min\{i, n-i\}} \binom{i}{\ell} \binom{n-i}{\ell} p^{2\ell} (1-p)^{n-2\ell}.
\end{equation}
The case $\ell=1$ yields the claimed lower bound. For the upper bound we bound the above term, using $\binom{n}{k} \le n^k/(k!)$ to bound both binomial coefficients:
\begin{align*}
p^{(i)} \le\;& (1-p)^{n} \sum_{\ell=1}^{\min\{i, n-i\}} \frac{i^\ell(n-i)^\ell}{\ell!\ell!} \cdot  p^{2\ell} (1-p)^{-2\ell}\\
\le\;& (1-p)^{n} \sum_{\ell=1}^{\infty} \left(\frac{i(n-i)p^2}{(1-p)^2}\right)^\ell\\
=\;& (1-p)^{n} \frac{\frac{i(n-i)p^2}{(1-p)^2}}{1-\frac{i(n-i)p^2}{(1-p)^2}}.
\end{align*}
Applying $\frac{1}{1-x} = 1 + \frac{x}{1-x} \le 1+2x$ for $x \le 1/2$ to $x:= \frac{i(n-i)p^2}{(1-p)^2}$ in the above formula yields
\[
(1-p)^{n} \cdot \frac{i(n-i)p^2}{(1-p)^2} \cdot \left(1+\frac{2i(n-i)p^2}{(1-p)^2}\right)
\]
and hence the claimed upper bound.

The second statement follows from the upper bound and the fact that the offspring has Hamming distance~2 in the case~${\ell=1}$, i.\,e., with probability $i(n-i)p^2(1-p)^{n-2}$.
\end{proof}

\begin{proof}[Proof of Lemma~\ref{lem:jump-to-level-i}]
A search point with $i$ ones is created from a parent with $i-d < i$ ones if there is a value $\ell$ such that $d+\ell$ 0\nobreakdash-bits flip to 1 and $\ell$ 1\nobreakdash-bits flip to 0. The sought probability therefore is
\begin{align*}
& \max_{d \ge 1} \left( \sum_{\ell=0} \binom{n-i+d}{d+\ell} \binom{i-d}{\ell} p^{d+2\ell} (1-p)^{n-d-2\ell}\right)\\
\le\;& \max_{d \ge 1} \left( \sum_{\ell=0} \frac{(n-i+d)^{d+\ell}}{(d+\ell)!} \cdot \frac{(i-d)^\ell}{\ell!} \cdot p^{d+2\ell}\right)\\
=\;& \max_{d \ge 1} \left( \sum_{\ell=0} \frac{(p(n-i+d))^{d}}{(d+\ell)!} \cdot \frac{(p^2 \cdot (i-d)(n-i+d))^\ell}{\ell!}\right)\\
\le\;& \max_{d \ge 1} \left( \frac{(p(n-i+d))^{d}}{d!} \right) \cdot \sum_{\ell=0} \frac{((pn)^2/4)^\ell}{\ell!}\\
=\;& \max_{d \ge 1} \left( \frac{(p(n-i+d))^{d}}{d!} \right) \cdot e^{(pn)^2/4}.
\end{align*}
Using $1/(d!) \le (e/d)^d$, we bound the $\max$ term as
\begin{align*}
\max_{d \ge 1} \left( \frac{(p(n-i+d))^{d}}{d!} \right)
\le\;& \max_{d \ge 1} \left( \frac{(ep(n-i+d))}{d} \right)^d\\
\le\;& \max_{d \ge 1} \left( ep(n-i+1) \right)^d.
\end{align*}
Now, if $ep(n-i+1) \le 1$, the maximum is attained for $d=1$, in which case we get a probability bound of
$
ep(n-i+1)
\cdot e^{(pn)^2/4}$ as claimed.
If $ep(n-i+1) > 1$ we trivially bound the sought probability by
\[
1 < ep(n-i+1) \le ep(n-i+1) \cdot e^{(pn)^2/4}.
\]
\end{proof}

\end{document}